\documentclass[mnsc]{informs3} 
\OneAndAHalfSpacedXI 

\usepackage{endnotes}
\usepackage[ruled]{algorithm}
\usepackage{algorithmicx,algpseudocode,amsfonts,subfigure}
\usepackage{graphicx,hyperref,comment,appendix}
\usepackage[capitalize,noabbrev]{cleveref}
\usepackage{tikz}
\usepackage{subfigure,caption,microtype,booktabs} 

\usepackage{amsthm,amsmath,amssymb,soul}
\usepackage{comment,enumitem}
\setlist{leftmargin=*}

\newtheorem{theorem}{Theorem}

\newtheorem{lemma}{Lemma}
\newtheorem{proposition}{Proposition}
\newtheorem{remark}{Remark}

\renewcommand{\Re}{\mathbb{R}}
\newcommand{\ie}{\textit{i.e.}}
\newcommand{\eg}{\textit{e.g.}}
\newcommand{\Cov}{\text{Cov}}
\newcommand{\E}{\mathbb{E}}
\renewcommand{\Pr}{\text{Pr}}
\newcommand{\Var}{\text{Var}}

\newcommand{\cE}{\mathcal E}
\newcommand{\istar}{i^*}
\newcommand{\eps}{\varepsilon}

\usepackage{natbib}
\bibpunct[, ]{(}{)}{,}{a}{}{,}%
\def\bibfont{\small}%
\usepackage{bm,comment}
\usepackage{soul}

\ECRepeatTheorems

\EquationsNumberedThrough    

\allowdisplaybreaks
\begin{document}


\RUNAUTHOR{}

\RUNTITLE{Best-Arm Identification with Generative Proxy}



\TITLE{Best-Arm Identification with Generative Proxy}

\ARTICLEAUTHORS{

\AUTHOR{Tianyi Ma}
\AFF{School of Operations Research and Information Engineering, Cornell University\\
	 \EMAIL{tm693@cornell.edu}} 

\AUTHOR{Hanzhang Qin}
\AFF{Department of Industrial Systems Engineering and Management, National University of Singapore \\
    \EMAIL{hzqin@nus.edu.sg}}

\AUTHOR{Ruihao Zhu}
\AFF{SC Johnson College of Business, Cornell University\\
	\EMAIL{ruihao.zhu@cornell.edu}} 

\AUTHOR{Jierui Zuo}
\AFF{Michael G. Foster School of Business, University of Washington\\
\EMAIL{jzuo4@uw.edu}}

} 

\ABSTRACT{%
Best-arm identification is a canonical model for data-driven decision-making, but in many applications each reward observation is costly. Motivated by the growing availability of cheap predictions from machine learning and large language models, we study fixed-confidence best-arm identification in which each costly reward pull is paired with a cheap but correlated proxy score. The marginal mean of the proxy can be estimated offline and is treated as known, whereas its correlation $\rho$ with the reward, which governs how much the proxy helps, is unknown and must be learned online in pair with real rewards. We show that a control-variate adjustment turns this model into a heteroscedastic identification problem whose oracle sample complexity improves by residual variance $1-\rho^2$. The central difficulty is that the correlation must be learned from the same costly samples that identification consumes online, and that a plug-in estimate of the residual variance is anti-conservative and can compromise correctness. We propose \texttt{PROBE} (\underline{PR}oxy \underline{O}LS for \underline{B}est-arm \underline{E}xploration), a phase-elimination algorithm that directly maintains an upper certificate on the residual variance with an ordinary least squares fit, whose exact $\chi^2$ law keeps the certificate valid regardless of the unknown correlation. We prove that \texttt{PROBE} is $\delta$-PAC and attains the known-correlation oracle sample complexity up to a constant multiplicative factor and a constant additive calibration cost. The guarantee extends to the $(\epsilon,\delta)$-PAC setting under minimal changes to the algorithm. Numerical experiments on synthetic instances and on an auto-loan pricing replay with large language model and tabular proxies confirm that the sample savings of \texttt{PROBE} scale with the strength of the reward-proxy correlation, exactly as the theory predicts.

}%


\KEYWORDS{online learning, generative model, dynamic pricing, control variates} 

\maketitle

%

\section{Introduction}\label{sec:intro}

Best-arm identification is a canonical formulation for decision making: among a finite set of alternatives, the decision maker seeks to identify the best one using as few samples as possible facing observation uncertainty. In many applications that motivate best-arm identification, these samples are online and expensive, \eg, a reward observation may require exposing a user to a recommendation, running a pricing experiment, or deploying a candidate policy in a real environment. Fortunately, cheap proxy evaluations for decision units can be produced with modern predictive and generative models. Such proxies are not reliable enough to replace the real reward, but they can be informative about its variation. Following the same philosophy of prediction-powered inference \citep{AngelopoulosEtAl2023}, the cheap proxy prediction can explain part of the uncertainty in an expensive target quantity and thereby reduce the number of costly observations needed for inference. In this paper, we study how to use such proxy scores as control variates in fixed-confidence best-arm identification. These proxy scores finds its place in many real world best-arm identification problems: 
\begin{itemize}
\item \textbf{Recommendation Systems:} In personalized recommendation, item cold-start is a tough problem. When a set of items (or themes) is newly launched to a user, the recsys needs to find which item (or theme) is the favorite of the user by click-through rate or watching time, facing uncertainty in user behavior due to many aspects. However, online samples are expensive as the user needs to be exposed to the items and usually only a few interactions are allowed. With the help of large language model (LLM), we are able to prompt the LLM with the item information, user interaction history, and all context information including session state, time and device, other recommended items, etc., to receive a proxy score, \eg, a score from 0 to 3 where 0 corresponds to no interest and 3 corresponds to very interested. During the preparation for the item launch, LLM can be called for many times for different user and different possible contexts, and we can acquire an accurate offline proxy score for each item, but the correlation between this proxy score and the real CTR or watching time remains unknown. Only after the items (or themes) are launched, the correlation can be estimated at the same time recsys explores the user's favorite. Due to the requirement of high concurrency and low latency, only a few or only one LLM call is available after each user interaction.

\item \textbf{Pricing Policy Selection:} A firm choosing among a set of candidate pricing rules, each mapping a customer to an offered price, need to identify the rule that yields the largest expected revenue. Here, online samples are expensive because each one requires offering a price to a real customer. A predictive model or an LLM-based scorer can instead produce a cheap proxy score for any customer-policy pair. Before the pricing experiment is deployed, this scorer can be queried offline on a large pool of historical contexts, so that the average proxy score of each policy is estimated to high accuracy. The correlation between the proxy score and the realized revenue, however, is not available offline and can only be estimated once the experiment is running. Moreover, because pricing decisions must be made in real time, only a few proxy evaluations, can be afforded for each realized customer interaction.

\end{itemize}

These two examples share the same structure. The proxy is correlated with the real reward and is therefore useful for variance reduction, serving to identify the best arm with fewer costly observations. Underlying both examples is an asymmetry between two kinds of proxy information: proxy-only information is cheap and abundant before the experiment begins, whereas reward-paired proxy information is scarce and must be acquired online. On the offline side, proxy-only scores can be generated in large quantities before the experiment by repeatedly querying the generator or predictive model at essentially no cost, so the marginal proxy mean can be estimated to arbitrary accuracy and treated as known. This does not require the proxy to be calibrated to the true reward mean. On the online side, the proxy information that actually reduces variance is the part paired with the same unit that generates the realized reward, and such paired evaluations are far more costly than offline ones, limited by latency and scalability. 

This asymmetry is the source of both the opportunity and the difficulty in our setting. The opportunity is clear: had the reward-proxy correlation been known in advance, a classical control-variate adjustment by \cite{AngelopoulosEtAl2023} would remove the predictable part of the reward noise, leaving each arm with a smaller residual variance and saving a large number of online samples in proxy-aided best-arm identification. The difficulty is that this correlation, unlike the proxy mean, must be estimated from the same costly reward-proxy paired samples that the identification is trying to economize. This naturally suggests learning the correlation as a first step and reducing the residual noise accordingly. However, aiming at the correlation is myopic. Estimating the correlation and proceeding with the estimated result may trust the observed reward more than warranted and commit to a decision before the arms are reliably separated, which jeopardizes correctness and not merely efficiency. In fact, simply being cautious about the correlation still does not repair this, because a guarantee on the correlation translates poorly into a guarantee on the variance reduction, and the translation degrades exactly when the proxy is most informative. In this paper, we show that such a cheap but unreliable proxy can be turned into a provable reduction in the number of costly online samples needed for best-arm identification.

\subsection{Main Contribution}
In this work, we formulate proxy-augmented fixed-confidence best-arm identification and propose \texttt{PROBE}, \underline{PR}oxy \underline{O}LS for \underline{B}est-arm \underline{E}xploration, an algorithm that turns an unreliable but cheap proxy into provable savings in costly online samples with reward-proxy correlation. Our main contributions are summarized as follows.
\begin{itemize}
\item In Section~\ref{sec:problem_formulation}, we formalize a best-arm identification problem in which each costly reward pull is paired with a cheap proxy score whose marginal mean is known but reward-proxy correlation is not. We show that a control-variate adjustment turns this model into a heteroscedastic BAI problem, in which the proxy reduces the effective noise variance of each arm  from one to the residual variance. This yields a clean oracle benchmark: were the correlation $\rho$ known, the sample complexity would improve by the arm-wise factor $1-\rho^2$. We further show that this single-proxy model is without loss of generality, as a batch of repeated proxy queries at one pull reduces exactly to a single averaged proxy.

\item In Section~\ref{sec:algo}, we design \texttt{PROBE} (\underline{PR}oxy \underline{O}LS for \underline{B}est-arm \underline{E}xploration), a phase-elimination algorithm that learns the variance reduction while performing identification. The central difficulty is that the correlation, and hence the variance reduction it implies, is unknown and must be learned from the same costly samples; a plug-in estimate of the reduced variance is anti-conservative and can compromise correctness. \texttt{PROBE} instead maintains a one-sided \emph{residual-variance upper certificate} obtained from an ordinary least squares fit, whose exact $\chi^2$ law makes the certificate valid regardless of the unknown correlation. A calibration stage and monotone updates keep the certificate within a factor $1+\kappa$ of the true residual variance throughout, and a deliberate one-round lag decouples mean estimation from variance calibration so that the concentration of mean estimates remain Gaussian.

\item In Section~\ref{sec:sample-complexity}, we prove that \texttt{PROBE} is $\delta$-PAC and, with high probability, attains the oracle gap- and variance-dependent sample complexity up to a multiplicative factor $1+\kappa$ and an additive calibration cost of order $\widetilde O(K/\kappa^2)$ that is independent of the gaps and correlations. The parameter $\kappa$ admits a clean interpretation as the price of not knowing the correlation, trading the tightness of the variance certificate against the cost of establishing it. Fixing $\kappa$ to a small constant recovers the known-correlation oracle rate up to constants. We also show that \texttt{PROBE} extends to the $(\epsilon,\delta)$-PAC setting with a minimal change to the stopping rule, preserving the same improvement.

\item In Sections~\ref{sec:experiments} and~\ref{sec:autoloan-replay}, we validate \texttt{PROBE} on synthetic instances and on an auto-loan pricing replay environment. A paired Gaussian benchmark shows that \texttt{PROBE} closely tracks the known-correlation oracle across the whole range of correlations, and a nonlinear environment with proxies learned by standard regression models confirms that the sample saving is governed by the induced reward-proxy correlation rather than by how the proxy is produced. In the replay study, proxy scores generated by large language models and by a structured tabular predictor reduce the fixed-confidence stopping time by up to $61\%$ relative to a reward-only baseline, with the savings tracking the strength of the same-unit reward-proxy correlation.

\end{itemize}

\subsection{Related Work}
\label{subsec:related-work}
\paragraph{Fixed-confidence identification and pure-exploration objectives.}
Our work builds on fixed-confidence best-arm identification (BAI), where an algorithm adaptively samples arms and stops with a recommendation that is correct with probability at least (1-$\delta$). This objective is related to ranking and selection in stochastic simulation \citep{Bechhofer1954,KimNelson2001,ChenLinYucesanChick2000,FrazierPowellDayanik2008,HongFanLuo2021} and to classical BAI algorithms based on elimination, confidence bounds, and information-theoretic sampling \citep{MannorTsitsiklis2004,EvenDar2006,GabillonEtAl2012,AudibertBubeckMunos2010,Karnin2013,Jamieson2014,KaufmannCappeGarivier2016,GarivierKaufmann2016}. Related pure-exploration objectives include top-(m) identification, thresholding, PAC subset selection, and all-($\epsilon$)-good arms \citep{KalyanakrishnanEtAl2012,BubeckWangViswanathan2013,LocatelliGutzeitCarpentier2016,AbernethyAminZhu2016,MasonEtAl2020,LiMaHuaZhu2025,feng2024satisficing}. These works study identification from reward observations alone; we study how paired proxy observations reduce reward-sample complexity.

\paragraph{Model-generated feedback and LLM-assisted decisions.}
Our motivation comes from decision problems where model-generated signals are informative but unreliable. \citet{JiEtAl2025} study regret minimization with correlated surrogate rewards, and related work uses biased offline or model-assisted data to improve bandit learning \citep{YangTanCheung2025,ZhangZhuXie2025}. In experimentation, \citet{AoEtAl2026} study BAI with LLM judges audited by human feedback, and \citet{YeYoganarasimhanZheng2025} use LLMs to cold-start content experiments. Marketing and economics work similarly treats LLM responses as synthetic data and documents both useful preference signals and systematic deviations from humans \citep{BrandIsraeliNgwe2023,GoliSingh2024,GuiToubia2023,WuWangWangZheng2024,LudwigMullainathanRambachan2025,WangZhangZhang2026,YinXin2026}. Different from the regret-minimization setting of \citet{JiEtAl2025}, we study fixed-confidence identification, where the proxy enters through within-arm residualization and fresh-batch OLS mean estimation is separated from \(\chi^2\)-based residual-variance certification.

\paragraph{Structured information in BAI.}
Several pure-exploration models exploit additional structure, including arm-level variances \citep{LuTaoZhang2021}, linear or transductive features \citep{SoareLazaricMunos2014,FiezEtAl2019,JedraProutiere2020,RedaKaufmannDelahayeDuriez2021,RedaTirinzoniDegenne2021,RiveraTewari2024}, and delayed, correlated, or covariance-aware feedback \citep{GroverEtAl2018,GuptaJoshiYagan2021,SaadBlanchardVerzelen2023}. These works change the cross-arm sampling structure. Our proxy is paired with the reward from the same pull, so the useful dependence is within an arm and within a sampled unit.

\paragraph{Control variates and prediction-assisted inference.}
Our variance reduction is a control-variate mechanism \citep{Glasserman2003}, related to stochastic bandits with correlated side observations for regret minimization \citep{VermaHanawal2021,VermaEtAl2023}. It also connects to prediction-powered inference and proxy-data transfer, which use predictions to improve efficiency while correcting bias \citep{AngelopoulosEtAl2023,AngelopoulosDuchiZrnic2023,Bastani2021}. In fixed-confidence BAI, however, the learned residual variance determines future sampling effort. \texttt{PROBE} therefore uses one-sided OLS variance certificates rather than a plug-in estimate of the reward-proxy correlation.

\section{Problem Formulation}\label{sec:problem_formulation}

In this section, we introduce the fixed-confidence best-arm identification problem and the proxy-augmented observation model studied. We first recall the classical benchmark without proxies, then discuss the oracle benchmark in which the reward-proxy correlation is known in advance.

\subsection{BAI Baseline}
We consider a fixed-confidence best-arm identification (BAI) problem in finite-armed bandit with $K$ arms indexed by $i\in[K]=\{1,\dots,K\}$. At each time step $t=1,2,\dots$, the learner pulls an arm $i$ for the $l$-th time and receives a reward $X_{i,l}$. For each arm $i$, the sequence $\{X_{i,l}\}_{l\ge 1}$ are i.i.d. samples from normal distribution $\mathcal{N}(p_i,1)$ with unknown mean $p_i = \E[X_{i,1}]\in[0,1]$, and reward sequences are independent across arms with unit variance. Denote the unique best arm as $\istar=\arg\max_{1\le i\le K}p_i$, the reward gap can be defined for each suboptimal arm $i\ne\istar$ as $\Delta_i=p_{\istar}-p_i$, and we further let $\Delta_{\istar}=\min_{i\neq\istar}\Delta_i$ for notation consistency. Given $\delta\in(0,1)$ The objective of BAI is to find the best arm $\istar$ with probability at least $1-\delta$. To be specific, with single input $\delta$, we need to find an algorithm that adaptively samples arms, stops at time $T$, and returns a recommendation $\hat{i}$, and such algorithm is called $\delta$-PAC (probably approximately correct) if $\Pr(\hat i = \istar)\geq1-\delta$.

The sample complexity of such algorithms is typically governed by the gap-dependent quantity $\sum_{i=1}^K \Delta_i^{-2}$. A standard fixed-confidence BAI baseline to obtain such gap-dependent guarantees is by phase elimination with exponential-gap \citep{EvenDar2006,Karnin2013}. The algorithm maintains an active set of arms and runs in rounds $r=1,2,\ldots$, with a target accuracy $\varepsilon_r$ that decreases geometrically, for example $\varepsilon_r=2^{-r}$. In each round, every active arm is sampled enough times, \eg, $\tilde{O}(\varepsilon_r^{-2})$, to estimate its mean up to order
$\varepsilon_r$. The algorithm then selects a reference arm, either by a robust subroutine such as median elimination \citep{Karnin2013} or simply by the empirical maximum among active arms, and removes arms whose empirical means are worse than the reference value by a constant multiple of $\varepsilon_r$. Intuitively, a suboptimal arm with gap $\Delta_i$ survives only until the first round in which $\varepsilon_r$ is smaller than a constant multiple of $\Delta_i$. Then, such algorithm identifies the best arm with sample complexity of order
\begin{equation}
\sum_{i=1}^K
\frac{1}{\Delta_i^2}
\log\left(\frac{1}{\delta}\log\left(\frac{1}{\Delta_i}\right)\right),
\label{eq:PE-bound-base}
\end{equation}
up to universal constants and lower-order logarithmic terms.

\subsection{BAI with proxy scores}

We now augment each reward observation with a proxy score. At each time arm $i$ is pulled (for the $l$-th time), the learner observes a paired reward $(X_{i,l},Y_{i,l})$, where $X_{i,l}$ is the noisy version of the unknown real reward $p_i$ and $Y_{i,l}$ is the proxy score (or reward) on the the same arm $i$. The proxy score may be produced by a predictive model, a simulator, or an AI evaluator, \eg, a pre-trained or fine-tuned AI models like large language models (LLMs). For each arm \(i\), the pairs \(\{(X_{i,s},Y_{i,s})\}_{s\ge 1}\) are i.i.d., and the sequences are independent across arms. 

The proxy score is informative through its statistical dependence with the true reward: it can help estimate $p_i$ by its correlation with $X_{i,l}$, even if its own mean is not calibrated to the reward mean. Formally, for each arm $i$, the reward and proxy score as a bivariate Gaussian pair \citep{JiEtAl2025} as
\begin{equation}
\begin{pmatrix}
    X_{i,l}\\
    Y_{i,l}
\end{pmatrix}
\sim
\mathcal{N}\!\left(
    \begin{pmatrix}
        p_i\\
        m_i
    \end{pmatrix},
    \begin{pmatrix}
        1 & \rho_i\\
        \rho_i & 1
    \end{pmatrix}
\right),   \qquad \rho_i\in(-1,1). 
\label{eq:proxy-model}
\end{equation}
Here, $p_i=\E[X_{i,1}]$ is the unknown reward mean to be learned, while $m_i=\E[Y_{i,1}]$ is known and used only to center the proxy score, \eg, by define $W_{i,l} = Y_{i,l}-m_i$ we have $W_{i,l}$ with zero mean. The correlation coefficient $\rho_i=\Cov(X_{i,l},Y_{i,l})$ captures the quality of the proxy information, where larger $|\rho_i|$ means that the centered proxy $W_{i,l}$ carries more information about the reward noise. Importantly, we do not require $m_i=p_i$, or even that $m_i$ be close to $p_i$. The proxy may therefore be biased as a direct surrogate for reward, as long as its centered variation is informative about $X_{i,l}$.

\begin{remark}[Proxy asymmetry between offline and online]\label{rmk:asymmetry}
The formulation in \eqref{eq:proxy-model} reflects the offline-online asymmetry in proxy scores highlighted in Section~\ref{sec:intro}, which makes the proxy mean $m_i$ treated as known while leaving the correlation $\rho_i$ to be learned.

\noindent\textbf{Known proxy mean.} Before the online phase, proxy calls are cheap, so the generator or predictive model can be queried repeatedly on each arm to estimate the marginal mean $m_i$ to arbitrary accuracy, and we therefore treat $m_i$ as known. The correlation $\rho_i$, by contrast, is a joint property of the reward and the proxy on the same online unit due to shared underlying features or explicit contexts, and can only be estimated from the costly paired samples drawn during the online phase.

\noindent\textbf{Single proxy per pull.} During the online phase, latency and concurrency constraints permit only a restricted number of proxy calls per interaction. Assuming a single proxy per pull is nonetheless without loss of generality: when the proxy is repeatedly called independently, they can be averaged into one control variate, as we formalize in Proposition~\ref{prop:vector-proxy-reduction}.
\end{remark}

In bandit problems, the estimation of mean rewards plays an important role. With the help of proxy scores, the variance of estimation can be reduced through a control-variate adjustment \citep{Glasserman2003}. To be specific, for each arm $i$ with unknown mean $p_i$, we can define the adjusted observation
\begin{equation}
X_{i,l}-\beta(Y_{i,l}-m_i)
\label{eq:control_variate}
\end{equation}
based on observed arm reward $X_{i,l}$ and proxy score $Y_{i,l}$ for any coefficient $\beta\in\Re$.
\begin{proposition}\label{prop:define-var-reduce}
For any coefficient $\beta\in\Re$, the adjusted observation with proxy score is unbiased with variance
$$
\Var(X_{i,l}-\beta(Y_{i,l}-m_i)) = 1+\beta^2-2\rho_i\beta.
$$
The variance is minimized at $\beta=\rho_i$, in which case
$$
\Var(X_{i,l}-\rho_i(Y_{i,l}-m_i)) = 1 - \rho_i^2=:v_i.
$$
\end{proposition}

Note that in current model \eqref{eq:proxy-model} and \eqref{eq:control_variate}, we have assumed a single proxy per pull. Although in real world application, proxy calls can be repeated more than once, if the samples share the same underlying feature or explicit context, more proxy scores are redundant without additional correlation information. Otherwise, the averaged proxy is a sufficient summary of the whole proxy vector, discussed in detail in Section~\ref{sec:independent-proxy} in the appendix.

If $\rho_i$ is known before sampling, then the control-variate adjusted random variable $X_{i,l}-\rho_i(Y_{i,l}-m_i)$ is a Gaussian observation with mean $p_i$ and residual variance $v_i=1-\rho_i^2$. Consequently, the fixed-confidence best-arm identification problem with known-correlation proxy reduces to a heteroscedastic BAI problem in which arm $i$ has noise variance $v_i$ in place of one. To be specific, the variance of each arm $i$ is reduced from $1$ to $1-\rho_i^2$. Rerunning the phase-elimination argument behind \eqref{eq:PE-bound-base} with each arm's sampling times scaling by its own residual variance $v_i$, the baseline complexity with known proxy becomes
\begin{equation}
\sum_{i=1}^K
\frac{1-\rho_i^2}{\Delta_i^2}
\log\left(\frac{1}{\delta}\log\left(\frac{1}{\Delta_i}\right)\right),
\label{eq:PE-bound-known}
\end{equation}
up to universal constants and lower-order logarithmic terms, and serves as the oracle benchmark. This shows the informative value of a proxy: when a proxy can explain part of the reward fluctuation for the same sampled unit, the variance of remaining noise reduces from one to $1-\rho_i^2$ after subtracting this predictable component. As how much fluctuation is explained can be measured by the correlation $\rho_i$, a highly correlated proxy improves BAI by reducing the residual uncertainty in the estimation of mean reward.

Despite the $1-\rho_i^2$ reduction in sample complexity given in \eqref{eq:PE-bound-known}, the main challenge is that the correlation $\rho_i$, and hence the residual variance $1-\rho_i^2$ that the oracle schedule relies on, is not known in advance. As we will discuss in Remark~\ref{remark:compare-v-rho-estimation}, estimating $\rho_i$ and plugging it into $1-\hat\rho_i^2$ is not enough as $|\rho_i|$ might be overestimated and the sampling is not enough before an elimination decision. A conservative bound on $\rho_i$ cannot yields a reliable certificate for the residual variance either especially when $\rho_i$ is large. Our algorithm instead estimates and certifies the residual variance directly through a linear-regression formulation, which we develop next.

\section{\texttt{PROBE}: Proxy OLS for Best-arm Exploration}\label{sec:algo}
In this section, we propose our algorithm \texttt{PROBE} that learns the control-variate coefficient while maintaining conservative residual-variance upper certificates. The algorithm uses ordinary least squares (OLS) to estimate the reward mean and the reward-proxy correlation. The residual-variance estimate is then converted into a residual-variance upper certificate that determines future batch sizes.

\subsection{OLS Preliminaries}

Under~\eqref{eq:proxy-model}, and plugging into $W_{i,l} = Y_{i,l}-m_i$, we note that control-variate adjusted observation in \eqref{eq:control_variate} of each pull of arm $i$ can be written as the regression model
\[
X_{i,l}=p_i+\rho_i W_{i,l}+\varepsilon_{i,l},
\qquad
W_{i,l}\sim \mathcal{N}(0,1),\qquad
\varepsilon_{i,l}\sim \mathcal{N}(0,v_i),
\]
where $v_i=1-\rho_i^2$ and $W_{i,l}$ is independent of $\varepsilon_{i,l}$. Here, $\rho_i$ is the slope, and $p_i$ is the intercept in OLS. Fix arm $i$ and a fresh batch of size $s\ge 3$, with observations $\{(X_{i,l},W_{i,l})\}_{l=1}^s$, the OLS estimator is
\begin{equation}
(\hat p_i,\hat\rho_i)=\arg\min_{a,b}\sum_{l=1}^s (X_{i,l}-a-bW_{i,l})^2.
\label{eq:ols}
\end{equation}
Equivalently, with $\bar X_i=s^{-1}\sum_{l=1}^s X_{i,l}$ and
$\bar W_i=s^{-1}\sum_{l=1}^s W_{i,l}$,
\[
\hat\rho_i=
\frac{\sum_{l=1}^s(X_{i,l}-\bar X_i)(W_{i,l}-\bar W_i)}
{\sum_{l=1}^s(W_{i,l}-\bar W_i)^2},
\qquad
\hat p_i=\bar X_i-\hat\rho_i \bar W_i.
\]
Let
\[
h_i=\frac{1}{s}+\frac{\bar W_i^2}{\sum_{l=1}^s(W_{i,l}-\bar W_i)^2},
\]
and define the unbiased OLS residual-variance estimator for arm $i$ by
\begin{equation}
\hat v_i=\frac{1}{s-2}\sum_{l=1}^s(X_{i,l}-\hat p_i-\hat\rho_i W_{i,l})^2.
\label{eq:ols-var}
\end{equation}
Then, the OLS estimator $\hat{p}_i$ and $\hat{v}_i$ satisfy the following statistical backbone of the algorithm.
\begin{proposition}\label{prop:ols-concentration}
For any fixed arm $i$, conditional on $W_{i,1},\ldots,W_{i,s}$, the OLS intercept satisfies
\[
\hat p_i-p_i\sim \mathcal{N}(0,v_i h_i).
\]
Moreover,
\[
\frac{(s-2)\hat v_i}{v_i}\sim \chi^2_{s-2}.
\]
\end{proposition}

\begin{remark}[Estimate $\hat{v}_i$ instead of $\hat{\rho}_i$]
The $\chi^2$-distribution of $\hat{v}_i/v_i$ plays an important role in our algorithm that bypasses the challenge caused by estimating $\hat{\rho}_i$ to be plugged into $v_i$. To be specific, if we only estimate $\hat{\rho}_i$, then
\begin{itemize}
\item Overestimating ${\rho}_i$ is problematic: When $\rho_i$ (assumed to be positive without loss of generality) is overestimated by an error $e>0$, \ie, $\hat{\rho}_i=\rho_i+e$, we have $1-\hat{\rho}_i^2\approx v_i-2\rho e < v_i$. Because the correctness of $\delta$-PAC usually requires at least $\tilde{O}(v_i/\eps_r^2)$ samples to distinguish a reward gap of width $\eps_r$, the sample size based on $1-\hat{\rho}_i^2$ is less than required and the correctness then collapses due to lack of confidence. The underlying problem of naive plug-in is that $1-\hat{\rho}_i^2$ is biased estimator for the residual variance $v_i$, \ie, $\E[1-\hat{\rho}_i^2]=1-\rho_i^2-\Var(\hat{\rho}_i)<v_i$.

\item Conservative LCB of $\hat{\rho}_i$ also fails: If we simply use a lower confidence bound of $\hat{\rho}_i$ to avoid overestimating, sample complexity might explode when the proxy is informative. For example, we use $\hat{\rho}_i-\text{rad}$ to plug into the residual variance where $\text{rad}=\tilde{O}\left(s^{-1/2}\right)$ is the confidence radius under sample size $s$. With similar order of confidence provided by $\text{rad}$, $\rho_i-\text{rad}/2 \leq \hat{\rho}_i \leq \rho_i+\text{rad}/2$ holds and implies $\hat{\rho}_i-\text{rad}\leq\rho_i-\text{rad}/2$. Because we still want $1-(\hat{\rho}_i-\text{rad})^2$ close enough to the real residual variance $v_i$ from above (otherwise it cannot reduce sample complexity and might collapse to the case without proxy), we require, \eg, $1-(\hat{\rho}_i-\text{rad})^2\leq (1+\kappa)v_i$ then $1-({\rho}_i-\text{rad}/2)^2\leq (1+\kappa)v_i$. This implies $s\geq\tilde{O}(\rho_i^2\kappa^{-2}v_i^{-2})$. When $\rho_i$ is close to 1 and $v_i$ is close to 0, the sample size $s$ diverges to infinite for any fixed constant $\kappa$. 

\end{itemize}
Instead, the advantages of estimation and concentration on $\hat{v}_i$ are 
\begin{itemize}
\item  $\hat{v}_i$ is an unbiased estimator for $v_i$. 

\item The $\chi^2$-distribution provides a scale-free concentration for $\hat{v}_i/v$.
\end{itemize}

\label{remark:compare-v-rho-estimation}
\end{remark}

\subsection{The \texttt{PROBE} Algorithm}

For a batch of size $s\ge 3$, define the \textit{residual-variance upper certificate} given residual-variance estimation $\hat{v}$ as
\begin{equation}
U(s,\hat v,\eta)=
\begin{cases}
\displaystyle
\frac{\hat v}{1-2\sqrt{\frac{\log(1/\eta)}{s-2}}},
&\text{if }s-2>4\log(1/\eta),\\[2ex]
+\infty,
&\text{otherwise.}
\end{cases}
\label{eq:U_def}
\end{equation}
The lower tail of the $\chi^2$-distribution and Proposition~\ref{prop:ols-concentration} imply that $v_i\le U(s,\hat v_i,\eta)$ with probability at least $1-\eta$. The algorithm uses this residual-variance upper certificate rather than the plug-in estimate $\hat v_i$.

We now present \texttt{PROBE}, short for \underline{PR}oxy \underline{O}LS for \underline{B}est-arm \underline{E}xploration, in Algorithm~\ref{alg:probe}. The algorithm has two parts. First, it draws a calibration batch from every arm to initialize a conservative upper bound $U_{i,0}$ on the residual variance $v_i$. Second, it runs phased elimination. In each round $r$, active arms with smaller residual-variance upper certificates $U_{i,r}$ receive smaller fresh batches, and the OLS intercept from the fresh batch is used as the reward-mean estimate for elimination.
\begin{algorithm}[!ht]
\caption{\texttt{PROBE}}
\label{alg:probe}
\begin{algorithmic}
\State \textbf{Input:} $\delta$, $K$, $\kappa$
\State \textbf{Initialization:} $t_{\rm cal}\leftarrow2+\left\lceil{64\log(8K/\delta)}/{\kappa^2}\right\rceil$, $S_1\leftarrow\{1,\ldots,K\}$, $r\leftarrow 1$, $s_{i,r}\leftarrow1+t_{\rm cal}$.
\For{Each arm $i\in[K]$}
    \State Pull arm $i$ for $t_{\rm cal}$ times and compute the OLS residual variance $\hat{v}_{i,0}$.
    \State Set $U_{i,0}\leftarrow U(t_{\rm cal}, \hat{v}_{i,0}, \delta/(8K))$.
\EndFor

\While{$|S_r|>1$}
\State Set $\delta_r\leftarrow\delta/(16Kr^2)$, $\eps_r\leftarrow2^{-r}$.
\For{Each active arm $i\in S_r$}
    \State Set batch size $t_{i,r} \leftarrow \max\left\{3,
    \left\lceil 1+16\log({2}/{\delta_r})\right\rceil,
    \left\lceil 8\log({4}/{\delta_r})\right\rceil,
    \left\lceil 64U_{i,r-1}\log({2}/{\delta_r})\eps_r^{-2}\right\rceil \right\}$.
    \State Pull arm $i$ for $t_{i,r}$ times.
    \State Observe arm rewards $\{X_{i,s_{i,r}},\dots,X_{i,s_{i,r}+t_{i,r}-1}\}$ and proxy scores $\{Y_{i,s_{i,r}},\dots,Y_{i,s_{i,r}+t_{i,r}-1}\}$.
    \State Compute the OLS estimator
    $$
    (\hat{p}_{i,r},\hat{\rho}_{i,r})\leftarrow\arg\min_{a,b}\sum_{l=s_{i,r}}^{s_{i,r}+t_{i,r}-1}(X_{i,l}-a-b(Y_{i,l}-m_i))^2.
    $$
    \State Compute the residual variance $\hat{v}_{i,r}$
    $$
    \hat{v}_{i,r}\leftarrow\frac1{t_{i,r}-2}\sum_{l=s_{i,r}}^{s_{i,r}+t_{i,r}-1}(X_{i,l}-\hat{p}_{i,r}-\hat{\rho}_{i,r}(Y_{i,l}-m_i))^2
    $$
    \State Set $\widetilde U_{i,r} \leftarrow U(t_{i,r},\hat{v}_{i,r},\delta_r)$, $U_{i,r} \leftarrow \min\{U_{i,r-1},\widetilde U_{i,r}\}$, $s_{i,r+1}\leftarrow s_{i,r}+t_{i,r}$.
\EndFor

\State Let $i_r\in\arg\max_{i\in S_r}\hat{p}_{i,r}$, set $S_{r+1} \leftarrow S_r\setminus \left\{i\in S_r: \hat{p}_{i,r}<\hat{p}_{i_r,r}-\eps_r \right\}$.
\State $r\leftarrow r+1$.
\EndWhile
\end{algorithmic}
\end{algorithm}

\begin{remark}[Key design choices of \texttt{PROBE}]\label{remark:probe-novelty}
The correctness and efficiency of \texttt{PROBE} rest on the interplay of three design choices: a residual-variance upper certificate that replaces the plug-in estimate, a monotone update that keeps this certificate valid throughout, and a one-round lag that decouples mean estimation from variance calibration.

\noindent\textbf{Residual-variance upper certificate.} The upper certificate $U_{i,r}$ defined in \eqref{eq:U_def} is pivotal because it sandwiches the residual variance as $v_i\le U_{i,r-1}\le(1+\kappa)v_i$. The lower guarantee $v_i\le U_{i,r-1}$ delivers correctness: the round-$r$ batch is then large enough to estimate the OLS intercept $\hat p_{i,r}$ to the target accuracy $\eps_r$, so no arm is eliminated before its mean has been estimated reliably. The upper guarantee $U_{i,r-1}\le(1+\kappa)v_i$ delivers efficiency: the batch size exceeds the oracle allocation $(1-\rho_i^2)\eps_r^{-2}$ by at most a factor $1+\kappa$, so the total sample complexity stays controlled. Using a plug-in estimate $\hat v_i$ in place of $U_{i,r}$ would forfeit the lower guarantee, and hence correctness, for the reasons detailed in Remark~\ref{remark:compare-v-rho-estimation}.

\noindent\textbf{Monotone certificate updates.} The sandwich $v_i\leq U_{i,r-1}\leq(1+\kappa)v_i$ on residual variance is established by induction from the calibration stage. Drawing $t_{\rm cal}=2+\lceil64\log(8K/\delta)/\kappa^2\rceil$ samples per arm gives $v_i\le U_{i,0}\le(1+\kappa)v_i$ at $r=0$ for all arms with high probability, using the $\chi^2$ concentration of $\hat v_i/v_i$ in Proposition~\ref{prop:ols-concentration}. Each subsequent batch produces a fresh residual-variance upper certificate $\widetilde U_{i,r}$, and the monotone update $U_{i,r}=\min\{U_{i,r-1},\widetilde U_{i,r}\}$ preserves both sides of the sandwich: taking the minimum of valid upper bounds keeps the lower guarantee $v_i\le U_{i,r}$, while never worsening the initial factor-$(1+\kappa)$ tightness. Hence $v_i\le U_{i,r}\le(1+\kappa)v_i$ holds throughout the algorithm.

\noindent\textbf{One-round lag between mean and variance certificate.} \texttt{PROBE} sizes the round-$r$ batch through $U_{i,r-1}$, computed from data collected before round $r$, while it estimates $\hat p_{i,r}$ only from the fresh round-$r$ observations. This deliberate lag decouples estimation on mean from residual-variance upper certificate: since the accuracy of $\hat p_{i,r}$ is certified by the fixed $U_{i,r-1}\ge v_i$ from earlier rounds, Proposition~\ref{prop:ols-concentration} applies directly conditional on the fresh proxy values, and concentration of $\hat p_{i,r}$ remains exactly Gaussian. Reusing a single batch for both purposes would instead couple them and complicate the concentration of $\hat p_{i,r}$ to follow a Student-$t$ law rather than a Gaussian one, as in \citet{JiEtAl2025}. Instead, the lag yields a sharper concentration without changing the leading sample complexity as the geometrically growing phase lengths $\eps_r^{-2}$ absorb the one-round delay.
\end{remark}

\section{Sample Complexity}\label{sec:sample-complexity}

In this section, we show that \texttt{PROBE} matches the gap-dependent sample complexity reduced by the residual variance $1-\rho_i^2$ up to the multiplicative factor $1+\kappa$ and an initial calibration cost.

\subsection{$\delta$-PAC Analysis}
\begin{theorem}\label{thm:probe}
Under the proxy observation model~\eqref{eq:proxy-model}, \texttt{PROBE} is $\delta$-PAC. Moreover, with probability at least $1-\delta$, its total number of arm pulls is at most
\begin{equation}
\begin{aligned}
T\leq K\left(2+\left\lceil\frac{64\log(8K/\delta)}{\kappa^2}\right\rceil\right) +
\sum_{i=1}^K \Bigg(&\frac{1536(1+\kappa)(1-\rho_i^2)}{\Delta_i^2} \log\left(\frac{12K}{\delta}\log_2\left(\frac{3}{2\Delta_i}\right)\right)\\
&+\log\left(\frac{3}{2\Delta_i}\right) \left(64\log\left(\frac{12K}{\delta}\log_2\left(\frac{3}{2\Delta_i}\right)\right) +6\right)\Bigg)
\end{aligned}
\label{eq:probe-sample-complexity}
\end{equation}
\end{theorem}

\begin{proof}[Proof Sketch]

The proof works on a clean event under which two properties hold uniformly over arms and rounds: every residual-variance upper certificate is valid and tight enough as
\[
v_i\le U_{i,r}\le (1+\kappa)v_i,
\]
and every OLS intercept used for elimination is accurate as
\[
|\hat p_{i,r}-p_i|\le \eps_r/4.
\]
The clean event follows from three parts. First, the calibration event gives $v_i\le U_{i,0}\le (1+\kappa)v_i$, then the monotone updates and the $\chi^2$-distribution gives $v_i\le U_{i,r}\le (1+\kappa)v_i$ for all round $r\geq0$ with induction. Second, the OLS intercept estimation $\hat{p}_{i,r}$ is concentrated at $p_i$. Third, the leverage event $h_{i,r}\le 2/t_{i,r}$ happens (established by the first three terms in $t_{i,r}$), which further (by the final term in $t_{i,r}$) implies $|\hat p_{i,r}-p_i|\le \eps_r/4$ with the OLS intercept estimation.

Conditioned on this clean event which happens with probability at least $1-\delta$, correctness follows from the standard phase-elimination argument. The best arm is never removed: if $\istar\in S_r$, then the empirical reference arm $i_r$ satisfies
\[
\hat p_{i_r,r}\le p_{\istar}+\eps_r/4,
\qquad
\hat p_{\istar,r}\ge p_{\istar}-\eps_r/4,
\]
so $\hat p_{i_r,r}-\hat p_{\istar,r}\le \eps_r/2<\eps_r$. Conversely, a suboptimal arm $i$ is eliminated once $\eps_r<2\Delta_i/3$ because
\[
\hat p_{\istar,r}-\hat p_{i,r} \ge \Delta_i-\eps_r/2 > \eps_r.
\]
Since the empirical reference arm has value at least $\hat p_{\istar,r}$, arm $i$ is removed if it is still active.

It remains to count samples. A suboptimal arm $i$ is sampled only until the first round in which $\eps_r<2\Delta_i/3$, namely $r\geq O(\log_2(1/\Delta_i))$, and the best arm is sampled only until the analogous round with $\Delta_{\istar}=\min_{i\ne\istar}\Delta_i$. Conditioned on the clean event, arm $i$ is sampled for
\[
t_{\rm cal} + \sum_{r=1}^{\log_2(1/\Delta_i)}t_{i,r} \leq \tilde{O}\left( \frac{1}{\kappa^2} + \sum_{r=1}^{\log_2(1/\Delta_i)} \frac{(1+\kappa)(1-\rho_i^2)}{4^{-r}} \right) \leq \tilde{O}\left( \frac{1}{\kappa^2} + \frac{(1+\kappa)(1-\rho_i^2)}{\Delta_i^2} \right)
\]
Summing this display for every arm $i$ yields the bound in \eqref{eq:probe-sample-complexity}. The full proof is deferred to Section~\ref{sec:proof-of-thm-probe} in the appendix.
\end{proof}

Note that, Algorithm~\ref{alg:probe} can be easily extended to the $(\epsilon,\delta)$-PAC setting. In the previous $\delta$-PAC setting, we focus on exact best-arm identification, where the algorithm must return $\istar$ with probability at least $1-\delta$. In many applications, it is sufficient to identify an arm whose mean reward is within a prescribed accuracy level $\epsilon>0$ of the best arm. Formally, an algorithm is called $(\epsilon,\delta)$-PAC if its recommendation $\hat i$ satisfies
\[
\Pr(p_{\hat i}\ge p_{\istar}-\epsilon) \geq 1-\delta.
\]
To accommodate this relaxation, the modification of \texttt{PROBE} is minimal. The only change is the stopping rule: the algorithm is terminated at round $\min\{r:\eps_r<{2\epsilon}/{3}\}$ even if $|S_r|>1$ and an arbitrary arm in $S_r$ is returned.

The sample complexity is obtained from Theorem~\ref{thm:probe} by replacing each gap $\Delta_i$ with the effective gap $\max\{\Delta_i,\epsilon\}$. In particular, with probability at least $1-\delta$, the total number of arm pulls is at most
\begin{equation}
\begin{aligned}
T_{\epsilon}\leq K & \left(2+  \left\lceil\frac{64\log(8K/\delta)}{\kappa^2}\right\rceil\right) + \sum_{i=1}^K \Bigg( \frac{1536(1+\kappa)(1-\rho_i^2)}{\max\{\Delta_i,\epsilon\}^2} \log\left( \frac{12K}{\delta} \log_2\left(\frac{3}{2\max\{\Delta_i,\epsilon\}}\right) \right)\\
&+ \log_2\left(\frac{3}{2\max\{\Delta_i,\epsilon\}}\right) \left( 64\log\left( \frac{12K}{\delta} \log_2\left(\frac{3}{2\max\{\Delta_i,\epsilon\}}\right) \right)+6 \right) \Bigg).
\end{aligned}
\label{eq:probe-ep-delta}
\end{equation}
Ignoring constants and lower-order logarithmic terms, \eqref{eq:probe-ep-delta} gives
\[
\widetilde O\left(
\frac{K}{\kappa^2}
+
\sum_{i=1}^K
\frac{(1+\kappa)(1-\rho_i^2)}{\max\{\Delta_i,\epsilon\}^2}
\right),
\]
preserving the same variance-adaptive improvement as the exact-identification guarantee.

\subsection{Comparison with Benchmarks}

We then compare the sample complexity of \texttt{PROBE} in Theorem~\ref{thm:probe} with two benchmarks given in \eqref{eq:PE-bound-base} and \eqref{eq:PE-bound-known}. The sample complexity in Theorem~\ref{thm:probe} can be expressed in leading dependence
\begin{equation}
\sum_{i=1}^K \left( \frac{(1+\kappa)(1-\rho_i^2)}{\Delta_i^2}\log\left(\frac{K}{\delta}\log\left( \frac{1}{\Delta_i}\right) \right) + \frac{\log(K/\delta)}{\kappa^2} \right)
\end{equation}
where the first term corresponds to the phase elimination stage and the second terms corresponds to the calibration stage. Recall that the sample complexity of the fixed-confidence best-arm identification without proxy has leading dependence
\[
\sum_{i=1}^K
\frac{1}{\Delta_i^2}
\log\left(\frac{1}{\delta}\log\left(\frac{1}{\Delta_i}\right)\right).
\]
and the one with proxy and known correlation $\rho_i$ has leading dependence
\[
\sum_{i=1}^K
\frac{1-\rho_i^2}{\Delta_i^2}
\log\left(\frac{1}{\delta}\log\left(\frac{1}{\Delta_i}\right)\right),
\]
\texttt{PROBE} pays only a multiplicative factor $1+\kappa$ in the main gap-dependent term and an additive calibration cost of order $\tilde{O}(\kappa^{-2})$.

The parameter $\kappa$ captures the exploration and exploitation tradeoff facing unknown correlation: a tradeoff between the accuracy in residual variance $v_i$ estimation (exploitation) and the sample size in calibration for this accuracy (exploration). On one hand, $\kappa$ quantifies the relative accuracy with which the algorithm learns an upper certificate $U_{i,r}$ for $v_i$. To be specific, \texttt{PROBE} maintains $U_{i,r}\le (1+\kappa)v_i$ throughout the algorithm and provides tolerance towards overestimating $U_{i,r}$. This tolerance in upper bound further ensures that the batch size is not much larger than the oracle batch size based on the unknown residual variance $v_i=1-\rho_i^2$ by a factor of $1+\kappa$ during elimination. On the other hand, a lower tolerance in estimation accuracy requires more samples by an additive factor of $\kappa^{-2}$ to establish the certificate reliably during calibration. In a word, $\kappa$ represents the price of not knowing $\rho_i$ and quantifies the tradeoff between exploring and exploiting this correlation.

Thanks to the $\chi^2$-distribution of the residual variance, the additive exploration term is independent of the gaps $\Delta_i$ and the correlations $\rho_i$. We can fix it as a small constant value then the instance enjoys the same leading dependence as the known-$\rho_i$ oracle $\widetilde O\left( {(1-\rho_i^2)}/{\Delta_i^2}\right)$ for each arm $i$, up to a constant multiplicative factor.

\section{Numerical Experiments on Synthetic Data}\label{sec:experiments}
We use synthetic experiments to isolate the statistical mechanism behind \texttt{PROBE}: a proxy is useful when its centered within-arm variation explains reward noise, even if its mean is biased or rank-reversed. We study two settings. The first is a paired Gaussian benchmark in which the proxy is prescribed by hand and its correlation with the reward is a controlled parameter. The second is a nonlinear environment in which the proxy is a learned predictor, so that its correlation is an outcome of model fit rather than a chosen quantity. Throughout, the performance metric is the fixed-confidence stopping time, measured by the total number of online reward pulls before stopping.

\subsection{Gaussian Instance}\label{ssec:gaussian-benchmark}

\noindent\textbf{Setup:}
We consider a synthetic Gaussian instance with \(K=8\) arms. The reward means are $(1.00,0.90,0.84,0.76,0.70,0.62,0.54,0.44)$, and the proxy means are $(0.00,0.25,0.50,0.75,1.00,1.25,1.50,1.75)$, following model \eqref{eq:proxy-model}. Thus, the proxy means are deliberately ordered in the opposite direction from the reward means, so that the proxy cannot be useful through mean ranking alone. For each arm \(i\), we generate paired Gaussian observations with reward variance one, proxy variance one, and common reward-proxy correlation \(\rho\). We vary $\rho \in \{0,0.2,0.4,0.6,0.8,0.9\}$.

We compare three classes of methods. The first is \texttt{PROBE} without proxy, which ignores the proxy and uses reward observations only and the calibration step in Algorithm~\ref{alg:probe} is skipped. The second is a known-correlation oracle, which residualizes the reward using the true value of \(\rho\) and runs the same phase-elimination procedure with residual variance \(1-\rho^2\). The third is \texttt{PROBE} with unknown correlation, and we test among $\kappa \in \{0.1,0.2,0.5,1.0\}$. We set \(\delta=0.05\), repeat each configuration 1000 times, and report the average stopping time and its ratio relative to \texttt{PROBE} witout proxy at the same value of \(\rho\).

\noindent\textbf{Results:}
The results are shown in Figure~\ref{fig:gaussian-benchmark}. When \(\rho=0\), the proxy carries no information about the reward noise. In this case, proxy-assisted \texttt{PROBE} pays a learning overhead. This is consistent with the role of calibration, and when there is no variance reduction to learn, the proxy can only add estimation cost leading to a sample complexity even higher than not using proxy. As \(\rho\) increases, the proxy becomes more useful, the stopping time decreases, and exceed the no proxy baseline after $\rho\geq0.4$. Note that, although for all $\rho$ the sample complexity of \texttt{PROBE} is higher than the known-correlation baseline, due to the multiplicative factor $1+\kappa$, \texttt{PROBE} is very close to the known-proxy baseline and the $1-\rho^2$ theoretical reference.

These results confirm that the improvement comes from same-unit reward-proxy covariance rather than from the proxy mean itself. Although the proxy means are non-rank-preserving, \texttt{PROBE} substantially reduces the fixed-confidence stopping time once the proxy explains a nontrivial portion of the reward noise. The empirical error rates remain close to the nominal level, with the largest observed error rate below \(5\%\) across all configurations.

\begin{figure}[!ht]
    \centering
    \includegraphics[width=\linewidth]{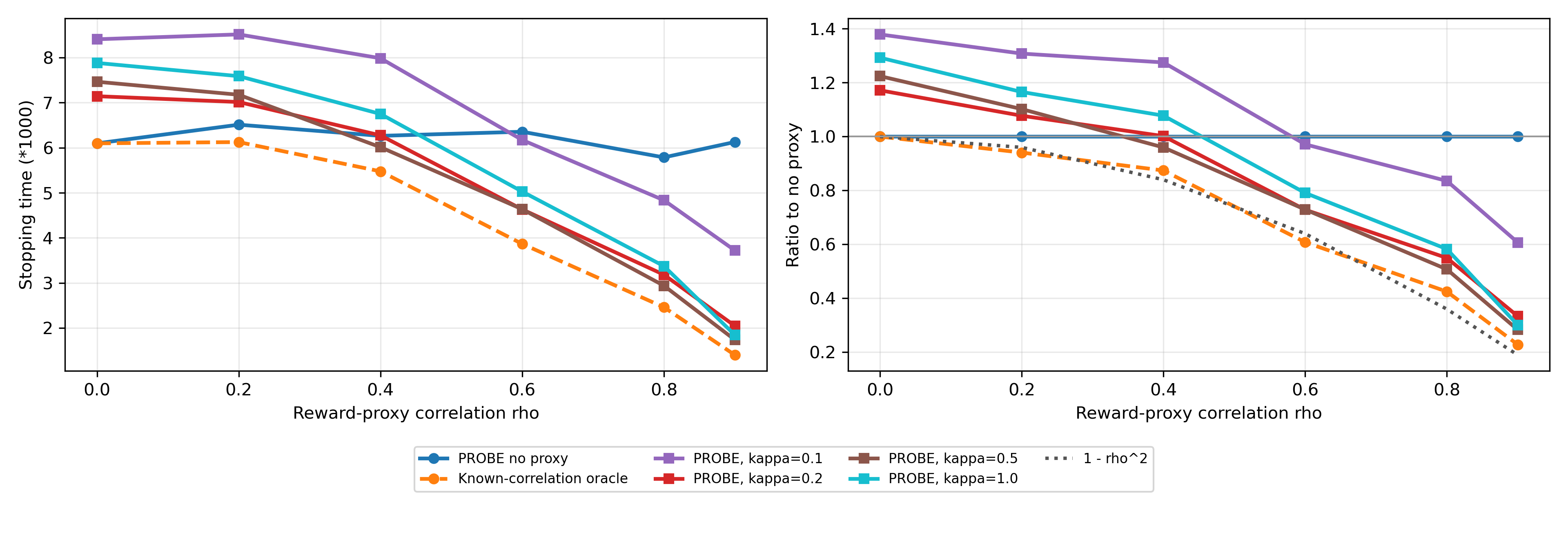}
    \caption{Sample Complexity under Gaussian Instance. The left panel plots stopping time measured in thousands, and the right panel plots stopping time normalized by reward-only \texttt{PROBE}.}
    \label{fig:gaussian-benchmark}
\end{figure}

\subsection{Nonlinear Instance}\label{ssec:ml-proxy}

\textbf{Setup:}
We consider \(K=5\) arms, where arm \(i\) corresponds to an action \(a_i\) on an evenly spaced grid, \(a_i\in\{-1,-0.5,0,0.5,1\}\). Each sampled unit carries a latent feature vector \(u=(u_1,\dots,u_5)\) with independent standard normal entries. Given action \(a\) and features \(u\), the reward is driven by the nonlinear signal
\begin{equation*}
g(u,a)=0.9\sin(u_1+a)+0.55\,a\,u_2+0.35\,(u_3^2-1)(a+0.4)+0.4\cos(u_4-0.5a)+0.2\,u_5,
\label{eq:ml-signal}
\end{equation*}
which mixes sinusoidal, quadratic, and action-feature interaction terms. The realized reward of arm \(i\) on a unit is
\begin{equation*}
X_i = p_i + \frac{g(u,a_i)+0.45\,\xi-\mu_{a_i}}{\sigma_{a_i}},\qquad \xi\sim\mathcal N(0,1),
\label{eq:ml-reward}
\end{equation*}
where \(\xi\) is independent observation noise, and \((\mu_{a_i},\sigma_{a_i})\) are the mean and standard deviation of \(g(u,a_i)+0.45\xi\), estimated once from a large pretabulated sample. Subtracting \(\mu_{a_i}\) and dividing by \(\sigma_{a_i}\) standardizes each arm to unit reward variance and mean \(p_i\), matching the unit-variance model of the analysis. The arm means \(p=(1.00,0.90,0.84,0.76,0.70)\) decrease across arms, so arm~1 is the best. As in the Gaussian benchmark, the proxy means \(m=(0.00,0.30,0.60,0.90,1.20)\) increase across arms and are therefore non-rank-preserving.

The proxy is a regression model fitted to predict the reward from the feature map $\phi(u,a)=(u,a,au_1,au_2,a^2)$. This map deliberately omits the terms through which \(g\) actually depends on the features, namely the sinusoid \(\sin(u_1+a)\), the quadratic \(u_3^2\), the cosine \(\cos(u_4-0.5a)\), and the linear term \(u_5\). No model built on \(\phi\) can therefore reproduce the reward exactly, and the achievable reward-proxy correlation is limited by how much of \(g\) each model recovers from \(\phi\). We fit three standard models of increasing flexibility on an offline training set of \(12{,}000\) units: linear regression, a depth-limited decision tree, and gradient boosting. Each fitted model then scores a separate pool of \(40{,}000\) units per arm, and each score is paired with the realized reward on the same unit. We use \(\delta=0.05\) and \(\kappa=1.0\), repeat each configuration 120 times, and report the stopping-time ratio relative to \texttt{PROBE} without proxy, together with the mean and minimum armwise correlations of each proxy.

\noindent\textbf{Results:}
The three models differ only in how much of the nonlinear structure omitted from the feature map $\phi(u,a)$ they can recover, and Figure~\ref{fig:ml-proxy-sim} shows that this is exactly the order in which they shorten the stopping time. Gradient boosting, the most flexible model, recovers the most of the omitted structure in \eqref{eq:ml-signal} and is the most strongly aligned with the reward, so it attains the largest sample saving. The decision tree and linear regression recover less and are correspondingly less aligned, so their savings are smaller and close to each other. All three proxies still identify the best arm in every repetition. The ordering matches the Gaussian benchmark and the theory: a proxy that explains more within-arm reward variation leaves a lower OLS residual variance and a shorter fixed-confidence stopping time, regardless of how the proxy is produced or whether its mean is rank-preserving.

\begin{figure}[!ht]
    \centering
    \includegraphics[width=0.49\linewidth]{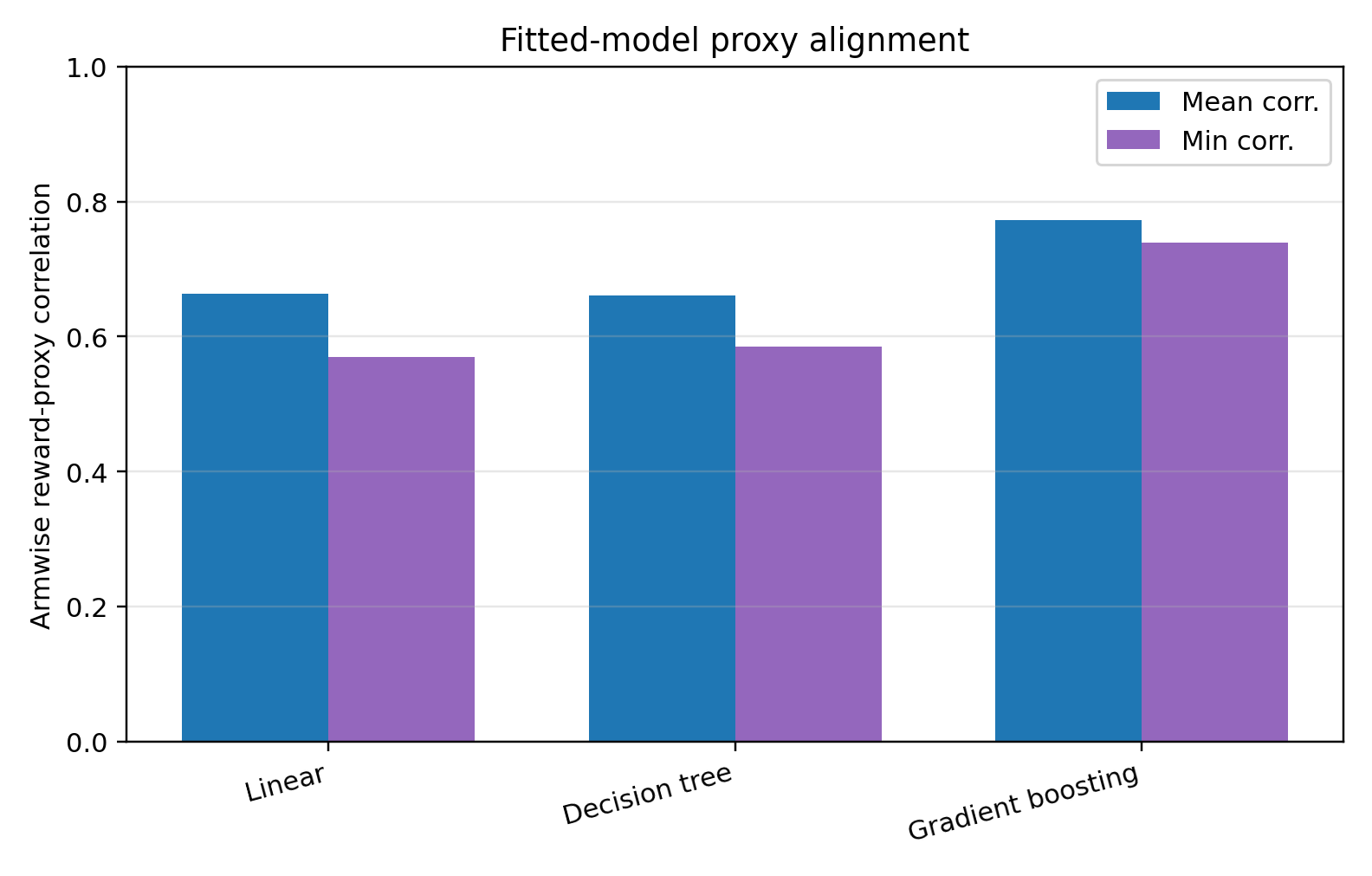}
    \includegraphics[width=0.49\linewidth]{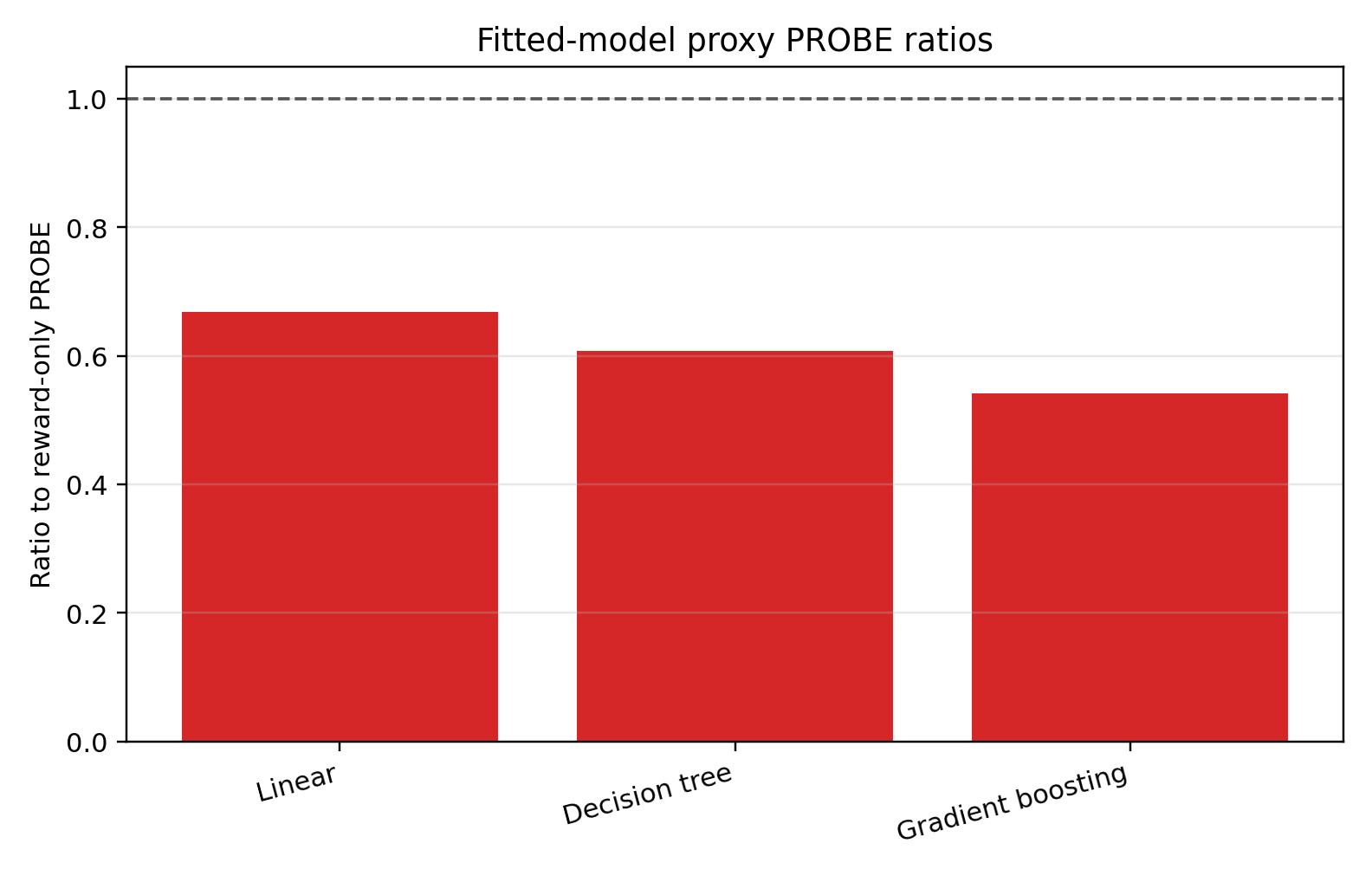}
    \caption{Nonlinear synthetic instance. The left panel reports the mean and minimum armwise reward-proxy correlations. The right panel reports the corresponding stopping-time ratio of \texttt{PROBE} relative to reward-only \texttt{PROBE}.}
    \label{fig:ml-proxy-sim}
\end{figure}

\section{Numerical Experiments on Real Data: Auto-loan Replay Experiment}
\label{sec:autoloan-replay}

We now evaluate \texttt{PROBE} on a pricing policy-selection task built from real auto-loan data. This is the setting of the pricing example in Section~\ref{sec:intro}: a lender must decide which of two pricing rules (two arms) earns more revenue, each online sample is a costly interaction with a real applicant, and a machine-learning model supplies a cheap proxy for the outcome. The experiment tests whether the variance reduction isolated in the synthetic study still materializes when the proxy is an LLM or a tabular predictor, and when the reward is a realized acceptance rather than Gaussian noise.

\subsection{Data and Setup}
We use the ``On-Line Auto Lending'' data set from the Center for Pricing and Revenue Management at Columbia University \citep{ColumbiaCPRM2012}, first studied by \citet{PhillipsSimsekVanRyzin2015}. Each record is an auto-loan application together with the contract offered to the applicant, described by customer and loan features (\eg, FICO score, approved amount, term, and competitor rate), the offered monthly-payment margin price, and a binary label indicating whether the applicant accepted the offer. We sort the applications chronologically and split them into a \emph{historical block} of the first $50{,}000$ records and a \emph{replay pool} of the next $40{,}000$ records. The historical block is used only to build the environment and to train the proxies, and the replay pool supplies the applicant contexts encountered online.

We fit an acceptance model $q(C,P)$ for each context $C$ and provided price $P$ on the historical block and treat it as the ground-truth demand environment. At each arm pull on the pricing rule, a fresh applicant context $C$ is drawn from the replay pool, and the pricing rule $\pi$ computes the price $P=\pi(C)$ to offer, which together generates a Bernoulli acceptance from $q$ and realized to 0 or 1. The realized reward is the margin revenue, equal to the offered price when the applicant accepts and zero otherwise, with mean reward
$$
m(\pi) = \E_{C\sim D_{\rm context}}[\pi(C)\cdot q(C,\pi(C))]
$$ 
as an expectation over context (applicant) distribution. Rewards are scaled to unit variance to match the normalization of the analysis. We emphasize that $q$ defines the environment only and is never revealed to the proxies.

We consider a best-arm identification on the two arms which are both fixed pricing rules. The historical rule $\pi_{\mathrm{hist}}$ re-offers the price actually observed for each applicant. The improved rule $\pi_{\mathrm{imp}}$ instead offers, for each applicant, the revenue-maximizing price among $\{0.8,0.9,1.0,1.1,1.2\}$ times the historical price within the empirical price support, where the maximization uses the fitted environment. Under the fitted environment, the historical rule earns expected margin revenue $133.0$ and the improved rule earns $177.7$, so the identification target is the improved rule. The fixed-confidence task is to recover this pricing rule (arm) with probability at least $1-\delta$.

Each proxy predicts an applicant's acceptance probability from the context and offered price, which we convert to a proxy revenue score by multiplying by the offered price. This score is paired with the realized reward on the same applicant-rule pull. Intuitively, the proxy helps whenever it tracks which applicants are unusually likely or unlikely to accept, so that residualizing the reward on the centered proxy removes part of the predictable, context-driven variation in revenue and leaves a smaller residual for the mean estimate. What remains after residualization, the idiosyncratic acceptance shock, cannot be predicted from the context and sets a floor on the achievable variance reduction.

We compare five variants of \texttt{PROBE}. The baseline is reward-only \texttt{PROBE}, which ignores the proxy. Three variants use LLM proxies: Qwen2.5-7B \citep{qwen2025qwen25technicalreport} with in-context learning run locally, GPT-5.5 \citep{openai2026gpt55} with in-context learning through its API, and Qwen2.5-7B fine-tuned on the historical block with LoRA, whose raw acceptance-classification score is used as the proxy. The last variant uses TabPFN \citep{hollmann2023tabpfntransformersolvessmall}, a pretrained tabular predictor, as a strong structured (non-LLM) proxy that calibrates how \texttt{PROBE} behaves when the proxy is highly accurate.

We run every variant with \(\delta=0.05\) and \(\kappa=1.0\), and average over $3{,}000$ independent replay repetitions. Besides the mean stopping time, we report its median and 90th percentile, the empirical correctness at stopping, the sample saving relative to reward-only \texttt{PROBE}, and the absolute reward-proxy correlation on each arm. We report absolute correlations because the sign of the association is absorbed by the OLS slope, so only its strength governs the residual-variance reduction.

\begin{table}[!ht]
\centering
\footnotesize
\setlength{\tabcolsep}{4pt}
\renewcommand{\arraystretch}{1.08}
\begin{tabular*}{\linewidth}{@{\extracolsep{\fill}}lrrrrrrr@{}}
\toprule
Method
& Mean
& Median
& 90\%
& Correct
& \(|\rho_{\mathrm{hist}}|\)
& \(|\rho_{\mathrm{imp}}|\)
& Saving \\
\midrule
Reward-only \texttt{PROBE}
& 105.2
& 105.2
& 106.7
& 1.000
& --
& --
& 0.0\% \\
Qwen2.5-7B local ICL
& 101.7
& 101.8
& 103.5
& 1.000
& 0.275
& 0.011
& 3.3\% \\
GPT-5.5 API ICL
& 95.6
& 95.7
& 97.6
& 1.000
& 0.264
& 0.327
& 9.1\% \\
Qwen2.5-7B LoRA raw CLS
& 52.0
& 52.3
& 54.5
& 1.000
& 0.725
& 0.707
& 50.5\% \\
TabPFN structured proxy
& 40.6
& 40.8
& 43.1
& 1.000
& 0.822
& 0.760
& 61.4\% \\
\bottomrule
\end{tabular*}
\caption{Stopping times (in thousands) of auto-loan replay across different proxies. The columns \(|\rho_{\mathrm{hist}}|\) and \(|\rho_{\mathrm{imp}}|\) report absolute reward-proxy correlations. Savings are relative to reward-only \texttt{PROBE}.}
\label{tab:auto-loan-results}
\end{table}

\begin{figure}[!ht]
    \centering
    \includegraphics[width=0.82\linewidth]{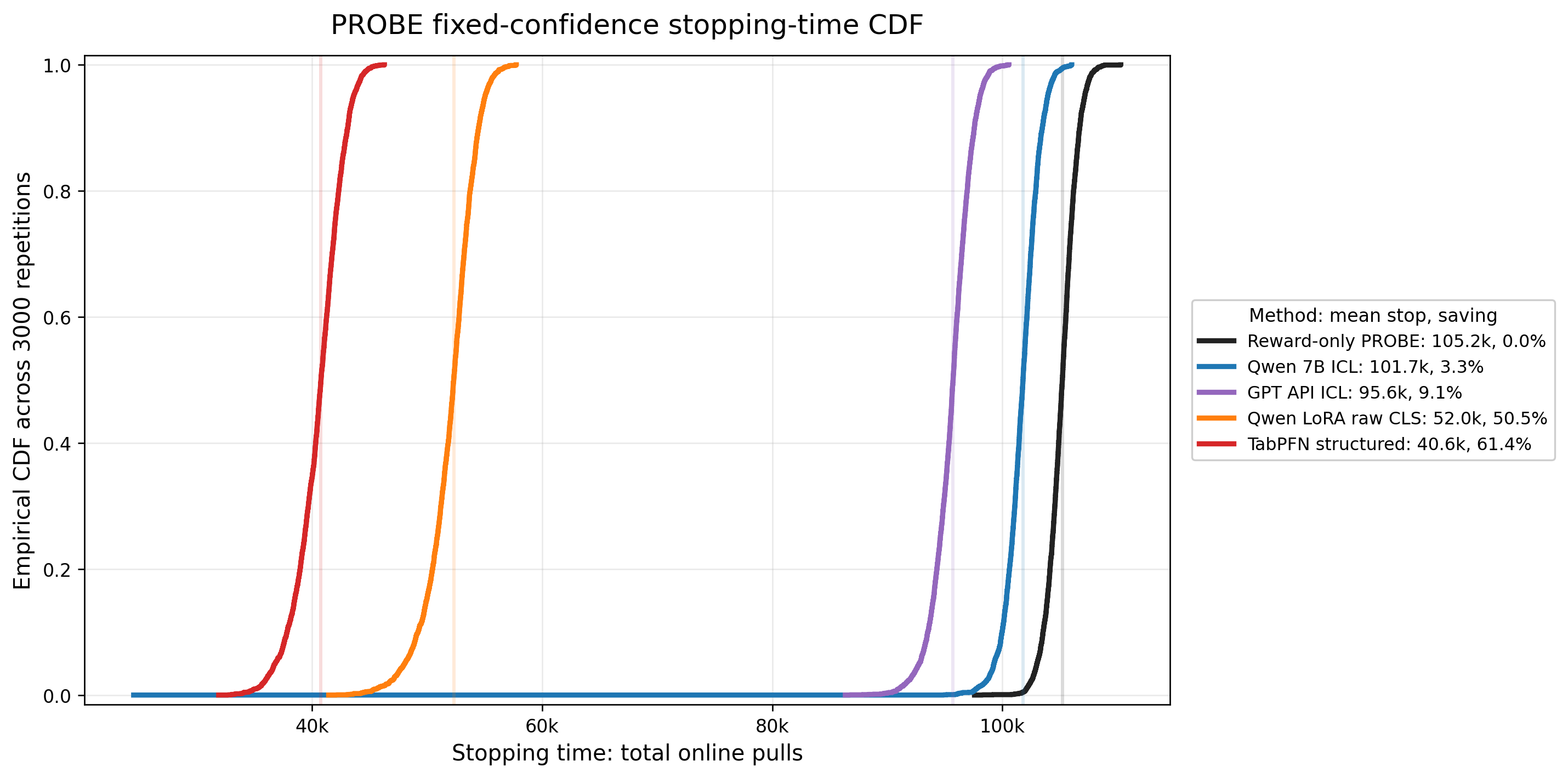}
    \caption{Empirical CDFs of fixed-confidence stopping times in the auto-loan replay experiment. Leftward shifts indicate fewer online reward pulls before \texttt{PROBE} stops at the target confidence level.}
    \label{fig:auto-loan-cdf}
\end{figure}

\subsection{Results}
Table~\ref{tab:auto-loan-results} and Figure~\ref{fig:auto-loan-cdf} report the results. Every variant identifies the improved rule in all $3{,}000$ repetitions, so all methods are exactly correct and none exhausts the sampling budget; the variants differ only in how quickly they stop. Reward-only \texttt{PROBE} stops after $105.2$k pulls on average. The two in-context-learning proxies help modestly, cutting the stopping time to $101.7$k pulls (Qwen2.5-7B, $3.3\%$ saving) and $95.6$k pulls (GPT-5.5, $9.1\%$). The fine-tuned and structured proxies help far more: Qwen2.5-7B LoRA reaches $52.0$k pulls ($50.5\%$ saving) and TabPFN reaches $40.6$k pulls ($61.4\%$). Figure~\ref{fig:auto-loan-cdf} shows the same ordering as a uniform leftward shift of the entire stopping-time distribution, not merely of its mean.

The savings track the strength of the reward-proxy correlation, exactly as the theory predicts. The in-context proxies are weakly and unevenly aligned with the reward: the local Qwen proxy correlates at $0.275$ on the historical arm but essentially zero ($0.011$) on the improved arm, so it can residualize only one of the two policies and saves little. The fine-tuned and structured proxies are strongly aligned on both arms ($0.725$ and $0.707$ for LoRA; $0.822$ and $0.760$ for TabPFN), so \texttt{PROBE} removes a large predictable component of revenue from each arm and needs far fewer pulls to separate the two policies. 

\section{Conclusion}\label{sec:conclusion}
In this paper, we study fixed-confidence best-arm identification when each costly reward pull is paired with a cheap proxy score whose marginal mean is known but whose correlation with the reward is not. We propose \texttt{PROBE}, a phase-elimination algorithm that uses the proxy as a control variate and maintains a residual-variance upper certificate, from an ordinary least squares fit, that stays valid regardless of the unknown correlation. We prove that \texttt{PROBE} is $\delta$-PAC and attains the known-correlation oracle sample complexity up to a factor $1+\kappa$ and an additive calibration cost of order $\widetilde O(K/\kappa^2)$, and that it extends to the $(\epsilon,\delta)$-PAC setting. Experiments on synthetic instances and on an auto-loan pricing replay with large language model and tabular proxies confirm that its sample savings scale with the strength of the reward-proxy correlation, as the theory predicts.

Our results suggest two directions for future study. First, the proxy here is fixed throughout identification; understanding how an online-refined proxy, for instance one fine-tuned as reward-proxy pairs accumulate, interacts with variance-adaptive sampling is a natural next step. Second, the same control-variate mechanism may reduce sampling cost in other pure-exploration objectives, such as top-$m$ identification and thresholding bandits.

\ACKNOWLEDGMENT{The authors gratefully acknowledge Columbia University Center for Pricing and Revenue Management for providing us the dataset on auto loans. We  express sincere gratitude to Qing Feng for valuable feedbacks.}

\bibliographystyle{ormsv080}
\bibliography{references}

\begin{thebibliography}{53}
\expandafter\ifx\csname natexlab\endcsname\relax\def\natexlab#1{#1}\fi
\expandafter\ifx\csname url\endcsname\relax
  \def\url#1{{\tt #1}}\fi
\expandafter\ifx\csname urlprefix\endcsname\relax\def\urlprefix{URL }\fi
\expandafter\ifx\csname urlstyle\endcsname\relax
  \expandafter\ifx\csname doi\endcsname\relax
  \def\doi#1{doi:\discretionary{}{}{}#1}\fi \else
  \expandafter\ifx\csname doi\endcsname\relax
  \def\doi{doi:\discretionary{}{}{}\begingroup \urlstyle{rm}\Url}\fi \fi

\bibitem[{Abernethy et~al.(2016)Abernethy, Amin, and Zhu}]{AbernethyAminZhu2016}
Abernethy, Jacob~D., Kareem Amin, Ruihao Zhu. 2016.
\newblock Threshold bandits, with and without censored feedback.
\newblock {\it Advances in Neural Information Processing Systems\/}, vol.~29. 4889--4897.

\bibitem[{Angelopoulos et~al.(2023{\natexlab{a}})Angelopoulos, Bates, Fannjiang, Jordan, and Zrnic}]{AngelopoulosEtAl2023}
Angelopoulos, Anastasios~N., Stephen Bates, Clara Fannjiang, Michael~I. Jordan, Tijana Zrnic. 2023{\natexlab{a}}.
\newblock Prediction-powered inference.
\newblock {\it Science\/} {\bf 382}(6671) 669--674.
\newblock \doi{10.1126/science.adi6000}.

\bibitem[{Angelopoulos et~al.(2023{\natexlab{b}})Angelopoulos, Duchi, and Zrnic}]{AngelopoulosDuchiZrnic2023}
Angelopoulos, Anastasios~N., John~C. Duchi, Tijana Zrnic. 2023{\natexlab{b}}.
\newblock {PPI++}: Efficient prediction-powered inference.
\newblock \doi{10.48550/arXiv.2311.01453}.

\bibitem[{Ao et~al.(2026)Ao, Chen, Gao, Li, and Simchi-Levi}]{AoEtAl2026}
Ao, Ruicheng, Hongyu Chen, Siyang Gao, Hanwei Li, David Simchi-Levi. 2026.
\newblock Best arm identification with {LLM} judges and limited human.

\bibitem[{Audibert et~al.(2010)Audibert, Bubeck, and Munos}]{AudibertBubeckMunos2010}
Audibert, Jean-Yves, S{\'e}bastien Bubeck, R{\'e}mi Munos. 2010.
\newblock Best arm identification in multi-armed bandits.
\newblock {\it Proceedings of the 23rd Annual Conference on Learning Theory\/}. 41--53.

\bibitem[{Bastani(2021)}]{Bastani2021}
Bastani, Hamsa. 2021.
\newblock Predicting with proxies: Transfer learning in high dimension.
\newblock {\it Management Science\/} {\bf 67}(5) 2964--2984.
\newblock \doi{10.1287/mnsc.2020.3729}.

\bibitem[{Bechhofer(1954)}]{Bechhofer1954}
Bechhofer, Robert~E. 1954.
\newblock A single-sample multiple decision procedure for ranking means of normal populations with known variances.
\newblock {\it The Annals of Mathematical Statistics\/} {\bf 25}(1) 16--39.
\newblock \doi{10.1214/aoms/1177728845}.

\bibitem[{Brand et~al.(2023)Brand, Israeli, and Ngwe}]{BrandIsraeliNgwe2023}
Brand, James, Ayelet Israeli, Donald Ngwe. 2023.
\newblock Using {LLMs} for market research.
\newblock \doi{10.2139/ssrn.4395751}.
\newblock Harvard Business School Marketing Unit Working Paper No. 23-062.

\bibitem[{Bubeck et~al.(2013)Bubeck, Wang, and Viswanathan}]{BubeckWangViswanathan2013}
Bubeck, S{\'e}bastien, Tengyao Wang, Nitin Viswanathan. 2013.
\newblock Multiple identifications in multi-armed bandits.
\newblock {\it Proceedings of the 30th International Conference on Machine Learning\/}, {\it Proceedings of Machine Learning Research\/}, vol.~28. PMLR, 258--265.

\bibitem[{Chen et~al.(2000)Chen, Lin, Y{\"u}cesan, and Chick}]{ChenLinYucesanChick2000}
Chen, Chun-Hung, Jianwu Lin, Enver Y{\"u}cesan, Stephen~E. Chick. 2000.
\newblock Simulation budget allocation for further enhancing the efficiency of ordinal optimization.
\newblock {\it Discrete Event Dynamic Systems\/} {\bf 10}(3) 251--270.
\newblock \doi{10.1023/A:1008349927281}.

\bibitem[{{Columbia University Center for Pricing and Revenue Management}(2012)}]{ColumbiaCPRM2012}
{Columbia University Center for Pricing and Revenue Management}. 2012.
\newblock {CPRM-12-001: On-Line Auto Lending Data Set}.
\newblock Data set, Center for Pricing and Revenue Management, Columbia Business School.
\newblock \urlprefix\url{https://business.columbia.edu/cprm-new/research}.
\newblock Accessed July 2, 2026.

\bibitem[{Even-Dar et~al.(2006)Even-Dar, Mannor, and Mansour}]{EvenDar2006}
Even-Dar, Eyal, Shie Mannor, Yishay Mansour. 2006.
\newblock Action elimination and stopping conditions for the multi-armed bandit and reinforcement learning problems.
\newblock {\it Journal of Machine Learning Research\/} {\bf 7}(39) 1079--1105.

\bibitem[{Feng et~al.(2024)Feng, Ma, and Zhu}]{feng2024satisficing}
Feng, Qing, Tianyi Ma, Ruihao Zhu. 2024.
\newblock Satisficing regret minimization in bandits: Constant rate and light-tailed distribution.
\newblock {\it arXiv preprint arXiv:2406.06802\/} .

\bibitem[{Fiez et~al.(2019)Fiez, Jain, Jamieson, and Ratliff}]{FiezEtAl2019}
Fiez, Tanner, Lalit Jain, Kevin~G. Jamieson, Lillian Ratliff. 2019.
\newblock Sequential experimental design for transductive linear bandits.
\newblock {\it Advances in Neural Information Processing Systems\/}, vol.~32. 10667--10677.

\bibitem[{Frazier et~al.(2008)Frazier, Powell, and Dayanik}]{FrazierPowellDayanik2008}
Frazier, Peter~I., Warren~B. Powell, Savas Dayanik. 2008.
\newblock A knowledge-gradient policy for sequential information collection.
\newblock {\it SIAM Journal on Control and Optimization\/} {\bf 47}(5) 2410--2439.
\newblock \doi{10.1137/070693424}.

\bibitem[{Gabillon et~al.(2012)Gabillon, Ghavamzadeh, and Lazaric}]{GabillonEtAl2012}
Gabillon, Victor, Mohammad Ghavamzadeh, Alessandro Lazaric. 2012.
\newblock Best arm identification: A unified approach to fixed budget and fixed confidence.
\newblock {\it Advances in Neural Information Processing Systems\/}, vol.~25. 3212--3220.

\bibitem[{Garivier and Kaufmann(2016)}]{GarivierKaufmann2016}
Garivier, Aur{\'e}lien, Emilie Kaufmann. 2016.
\newblock Optimal best arm identification with fixed confidence.
\newblock {\it Proceedings of the 29th Conference on Learning Theory\/}, {\it Proceedings of Machine Learning Research\/}, vol.~49. PMLR, 998--1027.

\bibitem[{Glasserman(2003)}]{Glasserman2003}
Glasserman, Paul. 2003.
\newblock {\it Monte Carlo Methods in Financial Engineering\/}, {\it Stochastic Modelling and Applied Probability\/}, vol.~53.
\newblock Springer, New York, NY.
\newblock \doi{10.1007/978-0-387-21617-1}.

\bibitem[{Goli and Singh(2024)}]{GoliSingh2024}
Goli, Ali, Amandeep Singh. 2024.
\newblock Frontiers: Can large language models capture human preferences?
\newblock {\it Marketing Science\/} {\bf 43}(4) 709--722.
\newblock \doi{10.1287/mksc.2023.0306}.

\bibitem[{Grover et~al.(2018)Grover, Markov, Attia, Jin, Perkins, Cheong, Chen, Yang, Harris, Chueh, and Ermon}]{GroverEtAl2018}
Grover, Aditya, Todor Markov, Peter Attia, Norman Jin, Nicolas Perkins, Bryan Cheong, Michael Chen, Zi~Yang, Stephen Harris, William Chueh, Stefano Ermon. 2018.
\newblock Best arm identification in multi-armed bandits with delayed feedback.
\newblock {\it Proceedings of the Twenty-First International Conference on Artificial Intelligence and Statistics\/}, {\it Proceedings of Machine Learning Research\/}, vol.~84. PMLR, 833--842.

\bibitem[{Gui and Toubia(2023)}]{GuiToubia2023}
Gui, George, Olivier Toubia. 2023.
\newblock The challenge of using {LLMs} to simulate human behavior: A causal inference perspective.
\newblock \doi{10.48550/arXiv.2312.15524}.

\bibitem[{Gupta et~al.(2021)Gupta, Joshi, and Yagan}]{GuptaJoshiYagan2021}
Gupta, Samarth, Gauri Joshi, Osman Yagan. 2021.
\newblock Best-arm identification in correlated multi-armed bandits.
\newblock {\it IEEE Journal on Selected Areas in Information Theory\/} {\bf 2}(2) 549--563.
\newblock \doi{10.1109/JSAIT.2021.3082028}.

\bibitem[{Hollmann et~al.(2023)Hollmann, Müller, Eggensperger, and Hutter}]{hollmann2023tabpfntransformersolvessmall}
Hollmann, Noah, Samuel Müller, Katharina Eggensperger, Frank Hutter. 2023.
\newblock {TabPFN}: A transformer that solves small tabular classification problems in a second.

\bibitem[{Hong et~al.(2021)Hong, Fan, and Luo}]{HongFanLuo2021}
Hong, L.~Jeff, Weiwei Fan, Jun Luo. 2021.
\newblock Review on ranking and selection: A new perspective.
\newblock {\it Frontiers of Engineering Management\/} {\bf 8}(3) 321--343.
\newblock \doi{10.1007/s42524-021-0152-6}.

\bibitem[{Jamieson et~al.(2014)Jamieson, Malloy, Nowak, and Bubeck}]{Jamieson2014}
Jamieson, Kevin, Matthew Malloy, Robert Nowak, S{\'e}bastien Bubeck. 2014.
\newblock lil' {UCB}: An optimal exploration algorithm for multi-armed bandits.
\newblock {\it Proceedings of the 27th Conference on Learning Theory\/}, {\it Proceedings of Machine Learning Research\/}, vol.~35. PMLR, 423--439.

\bibitem[{Jedra and Proutiere(2020)}]{JedraProutiere2020}
Jedra, Yassir, Alexandre Proutiere. 2020.
\newblock Optimal best-arm identification in linear bandits.
\newblock {\it Advances in Neural Information Processing Systems\/}, vol.~33. 10007--10017.

\bibitem[{Ji et~al.(2025)Ji, Pan, Zhu, and Lei}]{JiEtAl2025}
Ji, Wenlong, Yihan Pan, Ruihao Zhu, Lihua Lei. 2025.
\newblock Multi-armed bandits with machine learning-generated surrogate rewards.
\newblock \doi{10.48550/arXiv.2506.16658}.

\bibitem[{Kalyanakrishnan et~al.(2012)Kalyanakrishnan, Tewari, Auer, and Stone}]{KalyanakrishnanEtAl2012}
Kalyanakrishnan, Shivaram, Ambuj Tewari, Peter Auer, Peter Stone. 2012.
\newblock {PAC} subset selection in stochastic multi-armed bandits.
\newblock {\it Proceedings of the 29th International Conference on Machine Learning\/}. 655--662.

\bibitem[{Karnin et~al.(2013)Karnin, Koren, and Somekh}]{Karnin2013}
Karnin, Zohar, Tomer Koren, Oren Somekh. 2013.
\newblock Almost optimal exploration in multi-armed bandits.
\newblock {\it Proceedings of the 30th International Conference on Machine Learning\/}, {\it Proceedings of Machine Learning Research\/}, vol.~28. PMLR, 1238--1246.

\bibitem[{Kaufmann et~al.(2016)Kaufmann, Capp{\'e}, and Garivier}]{KaufmannCappeGarivier2016}
Kaufmann, Emilie, Olivier Capp{\'e}, Aur{\'e}lien Garivier. 2016.
\newblock On the complexity of best-arm identification in multi-armed bandit models.
\newblock {\it Journal of Machine Learning Research\/} {\bf 17}(1) 1--42.

\bibitem[{Kim and Nelson(2001)}]{KimNelson2001}
Kim, Seong-Hee, Barry~L. Nelson. 2001.
\newblock A fully sequential procedure for indifference-zone selection in simulation.
\newblock {\it ACM Transactions on Modeling and Computer Simulation\/} {\bf 11}(3) 251--273.
\newblock \doi{10.1145/502109.502111}.

\bibitem[{Li et~al.(2025)Li, Ma, Hua, and Zhu}]{LiMaHuaZhu2025}
Li, Zhekai, Tianyi Ma, Cheng Hua, Ruihao Zhu. 2025.
\newblock Identifying all {$\epsilon$}-best arms in (misspecified) linear bandits.
\newblock \doi{10.48550/arXiv.2510.00073}.

\bibitem[{Locatelli et~al.(2016)Locatelli, Gutzeit, and Carpentier}]{LocatelliGutzeitCarpentier2016}
Locatelli, Andrea, Maurilio Gutzeit, Alexandra Carpentier. 2016.
\newblock An optimal algorithm for the thresholding bandit problem.
\newblock {\it Proceedings of the 33rd International Conference on Machine Learning\/}, {\it Proceedings of Machine Learning Research\/}, vol.~48. PMLR, 1690--1698.

\bibitem[{Lu et~al.(2021)Lu, Tao, and Zhang}]{LuTaoZhang2021}
Lu, Pinyan, Chao Tao, Xiaojin Zhang. 2021.
\newblock Variance-dependent best arm identification.
\newblock {\it Proceedings of the Thirty-Seventh Conference on Uncertainty in Artificial Intelligence\/}, {\it Proceedings of Machine Learning Research\/}, vol. 161. PMLR, 1120--1129.

\bibitem[{Ludwig et~al.(2025)Ludwig, Mullainathan, and Rambachan}]{LudwigMullainathanRambachan2025}
Ludwig, Jens, Sendhil Mullainathan, Ashesh Rambachan. 2025.
\newblock Large language models: An applied econometric framework.
\newblock Working Paper 33344, National Bureau of Economic Research.
\newblock \doi{10.3386/w33344}.

\bibitem[{Mannor and Tsitsiklis(2004)}]{MannorTsitsiklis2004}
Mannor, Shie, John~N. Tsitsiklis. 2004.
\newblock The sample complexity of exploration in the multi-armed bandit problem.
\newblock {\it Journal of Machine Learning Research\/} {\bf 5} 623--648.

\bibitem[{Mason et~al.(2020)Mason, Jain, Tripathy, and Nowak}]{MasonEtAl2020}
Mason, Blake, Lalit Jain, Ardhendu Tripathy, Robert Nowak. 2020.
\newblock Finding all {$\epsilon$}-good arms in stochastic bandits.
\newblock {\it Advances in Neural Information Processing Systems\/}, vol.~33. 20707--20718.

\bibitem[{Ochoa~Rivera and Tewari(2024)}]{RiveraTewari2024}
Ochoa~Rivera, Eduardo, Ambuj Tewari. 2024.
\newblock Optimal thresholding linear bandit.
\newblock \doi{10.48550/arXiv.2402.09467}.

\bibitem[{{OpenAI}(2026)}]{openai2026gpt55}
{OpenAI}. 2026.
\newblock Introducing {GPT-5.5}.
\newblock \urlprefix\url{https://openai.com/index/introducing-gpt-5-5/}.

\bibitem[{Phillips et~al.(2015)Phillips, {\c{S}}im{\c{s}}ek, and van Ryzin}]{PhillipsSimsekVanRyzin2015}
Phillips, Robert, A.~Serdar {\c{S}}im{\c{s}}ek, Garrett van Ryzin. 2015.
\newblock The effectiveness of field price discretion: Empirical evidence from auto lending.
\newblock {\it Management Science\/} {\bf 61}(8) 1741--1759.
\newblock \doi{10.1287/mnsc.2014.2084}.

\bibitem[{{Qwen} et~al.(2025){Qwen}, Yang, Yang, Zhang, Hui, Zheng, Yu, Li, Liu, Huang, Wei, Lin, Yang, Tu, Zhang, Yang, Yang, Zhou, Lin, Dang, Lu, Bao, Yang, Yu, Li, Xue, Zhang, Zhu, Men, Lin, Li, Tang, Xia, Ren, Ren, Fan, Su, Zhang, Wan, Liu, Cui, Zhang, and Qiu}]{qwen2025qwen25technicalreport}
{Qwen}, An~Yang, Baosong Yang, Beichen Zhang, Binyuan Hui, Bo~Zheng, Bowen Yu, Chengyuan Li, Dayiheng Liu, Fei Huang, Haoran Wei, Huan Lin, Jian Yang, Jianhong Tu, Jianwei Zhang, Jianxin Yang, Jiaxi Yang, Jingren Zhou, Junyang Lin, Kai Dang, Keming Lu, Keqin Bao, Kexin Yang, Le~Yu, Mei Li, Mingfeng Xue, Pei Zhang, Qin Zhu, Rui Men, Runji Lin, Tianhao Li, Tianyi Tang, Tingyu Xia, Xingzhang Ren, Xuancheng Ren, Yang Fan, Yang Su, Yichang Zhang, Yu~Wan, Yuqiong Liu, Zeyu Cui, Zhenru Zhang, Zihan Qiu. 2025.
\newblock Qwen2.5 technical report.
\newblock \doi{10.48550/arXiv.2412.15115}.

\bibitem[{R{\'e}da et~al.(2021{\natexlab{a}})R{\'e}da, Kaufmann, and Delahaye-Duriez}]{RedaKaufmannDelahayeDuriez2021}
R{\'e}da, Cl{\'e}mence, Emilie Kaufmann, Andr{\'e}e Delahaye-Duriez. 2021{\natexlab{a}}.
\newblock Top-m identification for linear bandits.
\newblock {\it Proceedings of the 24th International Conference on Artificial Intelligence and Statistics\/}, {\it Proceedings of Machine Learning Research\/}, vol. 130. PMLR, 1108--1116.

\bibitem[{R{\'e}da et~al.(2021{\natexlab{b}})R{\'e}da, Tirinzoni, and Degenne}]{RedaTirinzoniDegenne2021}
R{\'e}da, Cl{\'e}mence, Andrea Tirinzoni, R{\'e}my Degenne. 2021{\natexlab{b}}.
\newblock Dealing with misspecification in fixed-confidence linear top-m identification.
\newblock {\it Advances in Neural Information Processing Systems\/}, vol.~34. 25489--25501.

\bibitem[{Saad et~al.(2023)Saad, Blanchard, and Verzelen}]{SaadBlanchardVerzelen2023}
Saad, El~Mehdi, Gilles Blanchard, Nicolas Verzelen. 2023.
\newblock Covariance-adaptive best arm identification.
\newblock {\it Advances in Neural Information Processing Systems\/}, vol.~36. 73287--73298.

\bibitem[{Soare et~al.(2014)Soare, Lazaric, and Munos}]{SoareLazaricMunos2014}
Soare, Marta, Alessandro Lazaric, R{\'e}mi Munos. 2014.
\newblock Best-arm identification in linear bandits.
\newblock {\it Advances in Neural Information Processing Systems\/}, vol.~27. 828--836.

\bibitem[{Verma et~al.(2023)Verma, Dai, Shu, and Low}]{VermaEtAl2023}
Verma, Arun, Zhongxiang Dai, Yao Shu, Bryan Kian~Hsiang Low. 2023.
\newblock Exploiting correlated auxiliary feedback in parameterized bandits.
\newblock {\it Advances in Neural Information Processing Systems\/}, vol.~36. 4430--4451.

\bibitem[{Verma and Hanawal(2021)}]{VermaHanawal2021}
Verma, Arun, Manjesh~K. Hanawal. 2021.
\newblock Stochastic multi-armed bandits with control variates.
\newblock {\it Advances in Neural Information Processing Systems\/}, vol.~34. 27592--27603.

\bibitem[{Wang et~al.(2026)Wang, Zhang, and Zhang}]{WangZhangZhang2026}
Wang, Mengxin, Dennis~J. Zhang, Heng Zhang. 2026.
\newblock Large language models for market research: A data-augmentation approach.
\newblock {\it Marketing Science\/} {\bf 0}(0).
\newblock \doi{10.1287/mksc.2025.0009}.
\newblock Articles in Advance.

\bibitem[{Wu et~al.(2024)Wu, Wang, Wang, and Zheng}]{WuWangWangZheng2024}
Wu, Yuhang, Yingfei Wang, Chu Wang, Zeyu Zheng. 2024.
\newblock Large language model enhanced machine learning estimators for classification.
\newblock \doi{10.48550/arXiv.2405.05445}.

\bibitem[{Yang et~al.(2025)Yang, Tan, and Cheung}]{YangTanCheung2025}
Yang, Le, Vincent Y.~F. Tan, Wang~Chi Cheung. 2025.
\newblock Best arm identification with possibly biased offline data.
\newblock {\it Proceedings of the Forty-First Conference on Uncertainty in Artificial Intelligence\/}, {\it Proceedings of Machine Learning Research\/}, vol. 286. PMLR, 4715--4730.

\bibitem[{Ye et~al.(2025)Ye, Yoganarasimhan, and Zheng}]{YeYoganarasimhanZheng2025}
Ye, Zikun, Hema Yoganarasimhan, Yufeng Zheng. 2025.
\newblock {LOLA}: {LLM}-assisted online learning algorithm for content experiments.
\newblock {\it Marketing Science\/} {\bf 44}(5) 995--1016.
\newblock \doi{10.1287/mksc.2024.0990}.

\bibitem[{Yin and Xin(2026)}]{YinXin2026}
Yin, Qichuan, Linwei Xin. 2026.
\newblock Synthetic but not infinite: How much {LLM}-generated data to use in market research.
\newblock \doi{10.2139/ssrn.6078686}.
\newblock Available at SSRN 6078686.

\bibitem[{Zhang et~al.(2025)Zhang, Zhu, and Xie}]{ZhangZhuXie2025}
Zhang, Yixuan, Ruihao Zhu, Qiaomin Xie. 2025.
\newblock Contextual online pricing with (biased) offline data.
\newblock {\it Advances in Neural Information Processing Systems\/}, vol.~38.

\end{thebibliography}
\OneAndAHalfSpacedXI
\newpage

\begin{appendices}{\Large \noindent\textbf{Appendix}}

\section{Proof of Theorem~\ref{thm:probe}}\label{sec:proof-of-thm-probe}
We prove the result on the clean event. We first define the clean initialization on $U_{i,0}$ for every arm $i$ as
\[
\cE_{\rm cal} = \left\{v_i\le U_{i,0}\le (1+\kappa)v_i\right\}.
\]
For every arm $i$ and round $r$, define the leverage event
\[
\cE^h_{i,r}=\left\{h_{i,r}\le \frac{2}{t_{i,r}}\right\}.
\]
Define the variance-certificate update event
\[
\cE^U_{i,r}=\left\{v_i\le \widetilde U_{i,r}\right\}.
\]
Finally, define the mean-estimation event 
\[
\cE^p_{i,r} = \left\{|\hat{p}_{i,r}-p_i|\le \sqrt{2v_ih_{i,r}\log(2/\delta_r)}\right\}.
\]
Combining these four events, we can define the global clean event
\[
\cE= \cE_{\rm cal} \cap \bigcap_{r\ge 1}\bigcap_{i=1}^K \left(\cE^h_{i,r}\cap\cE^U_{i,r}\cap\cE^p_{i,r}\right),
\]
where events for inactive arms are treated as the whole sample space. Then, $\cE$ happens with probability at least $1-\delta$ with proof given in Section~\ref{ssec:proof-lemma-clean}.
\begin{lemma}
\label{lem:clean}
The event $\cE$ implies bounded upper bound $U_{i,r}$ for all arm $i$ and round $r\geq0$ as
\begin{equation}
v_i\leq U_{i,r}\leq (1+\kappa)v_i
\label{eq:U_concentrate}
\end{equation}
and the concentration in estimation $\hat{p}_{i,r}$ for all arm $i$ and round $r\geq1$ as
\begin{equation}
|\hat{p}_{i,r}-p_i| \leq \varepsilon_r/4.
\label{eq:p_concentration}
\end{equation}
Furthermore, $\cE$ happens with probability
\[
\Pr(\cE)\ge 1-\delta.
\]
\end{lemma}

Then, we prove conditioned on clean event $\cE$. Denote for every arm $i$, 
\[
R_i=\min\left\{r\ge 1:\ 2^{-r}<\frac{2\Delta_i}{3}\right\},
\]
namely, $R_i=\lceil\log(3/(2\Delta_i))\rceil\leq2\log(3/(2\Delta_i))$ because $\Delta_i\leq 1$.

\noindent\textbf{Elimination correctness.}
First, the best arm is never removed conditioned on $\cE$. Suppose $\istar\in S_r$. Since $i_r$ maximizes $\hat{p}_{i,r}$ over $S_r$,
\[
\hat{p}_{i_r,r}\le p_{i_r}+\frac{\eps_r}{4}\le p_{\istar}+\frac{\eps_r}{4},
\]
and
\[
\hat{p}_{\istar,r}\ge p_{\istar}-\frac{\eps_r}{4}.
\]
Hence
\[
\hat{p}_{i_r,r}-\hat{p}_{\istar,r}\le \frac{\eps_r}{2}<\eps_r.
\]
Therefore the elimination condition $\hat{p}_{\istar,r}<\hat{p}_{i_r,r}-\eps_r$ cannot hold. By induction, the best arm remains active in every round.

Second, every suboptimal arm is removed once the round resolution is below its gap. Let $i\ne\istar$. If $r\ge R_i$, since the best arm is still active, we have
\[
\hat{p}_{\istar,r}-\hat{p}_{i,r} \ge \left(p_{\istar}-\frac{\eps_r}{4}\right)-\left(p_i+\frac{\eps_r}{4}\right) = \Delta_i-\frac{\eps_r}{2} > \eps_r.
\]
where the final inequality is based on $\eps_r\le \eps_{R_i}<{2\Delta_i}/{3}$.
Because $\hat{p}_{i_r,r}\ge \hat{p}_{\istar,r}$,
\[
\hat{p}_{i,r}<\hat{p}_{i_r,r}-\eps_r.
\]
Thus arm $i$ is eliminated in round $r\geq R_i$ if it is still active. Hence suboptimal arm $i$ is sampled only through round $R_i$, and the best arm is sampled only through round $R_{\istar}$.

\noindent\textbf{Sample complexity.}
The arm pull of each $i$ at round $r$ is upper bounded by
\begin{equation*}
\begin{aligned}
t_{i,r} &\leq \left\lceil\frac{64U_{i,r-1}\log({2}/{\delta_r})}{\eps_r^2}\right\rceil + \max\left\{ 3, \left\lceil1+16\log\frac{2}{\delta_r}\right\rceil, \left\lceil8\log\frac{4}{\delta_r}\right\rceil \right\}\\
&\leq \frac{64(1+\kappa)(1-\rho_i^2)\log({2}/{\delta_r})}{\eps_r^2} + 1 + 2 + 16\log(2/\delta_r)
\end{aligned}
\end{equation*}
where first inequality is based on the definition of $t_{i,r}$, the second inequality is $U_{i,r-1}\le (1+\kappa)v_i=(1+\kappa)(1-\rho_i^2)$ implied by $\cE$, and $\log(4/\delta_r)\le 2\log(2/\delta_r)$ and $\log({2}/{\delta_r})\ge \log(32K/\delta)>1$. Since $\log(2/\delta_r)$ is increasing in $r$,
\begin{equation*}
\begin{aligned}
\sum_{r=1}^{R_i}\frac{\log(2/\delta_r)}{\eps_r^2} &\le \log(2/\delta_r)\sum_{r=1}^{R_i}4^r \leq \frac{4}{3}\log(2/\delta_{R_i})4^{R_i}\\
&\leq \frac{12\log(2/\delta_{R_i})}{\Delta_i^2},
\end{aligned}
\end{equation*}
where the final inequality is based on the definition of $R_i$, where either $R_i=1$ or $2^{-(R_i-1)}\ge 2\Delta_i/3$ which implies $2^{-R_i}\ge {\Delta_i}/{3}$ in both cases with $\Delta_i\le 1$/. We further plug $R_i\leq2\log(3/(2\Delta_i))$ into the definition of $\delta_{R_i}$ as
$$
\log\frac{2}{\delta_{R_i}} = \log\frac{32KR_i^2}{\delta} \le \log\left( \frac{128K}{\delta} \left(\log_2\left(\frac{3}{2\Delta_i}\right)\right)^2 \right) \leq 2\log\left(\frac{12K}{\delta}\log_2\left(\frac{3}{2\Delta_i}\right)\right).
$$
The post-calibration samples of arm $i$ are at most
\[
\sum_{r=1}^{R_i}t_{i,r} \le
\frac{1536(1+\kappa)(1-\rho_i^2)}{\Delta_i^2} \log\left(\frac{12K}{\delta}\log_2\left(\frac{3}{2\Delta_i}\right)\right)
+\log\left(\frac{3}{2\Delta_i}\right) \left(64\log\left(\frac{12K}{\delta}\log_2\left(\frac{3}{2\Delta_i}\right)\right) +6\right).
\]
Adding the $t_{\rm cal}$ calibration samples for each arm proves the stated total sample bound.

\subsection{Proof of Lemma~\ref{lem:clean}}\label{ssec:proof-lemma-clean}
\begin{proof}
To prove Lemma~\ref{lem:clean}, we need two additional lemmas with proofs in Section~\ref{ssec:proof-lemma-leverage} and \ref{ssec:proof-lemma-calibration}.
\begin{lemma}
\label{lem:leverage}
For every $\eta\in(0,1)$, if the sample size $s\ge \max\{3, 1+16\log(2/\eta), 8\log(4/\eta)\}$, then the leverage factor follows 
\[
\Pr\left(h\le \frac2s\right)\ge 1-\eta.
\]
\end{lemma}

\begin{lemma}
\label{lem:calibration}
For each arm, draw $t_{\rm cal}$ samples and compute the OLS residual variance $\hat{v}_0$. Define $U_0=U(t_{\rm cal},\hat{v}_0,\delta/(8K))$, then
\[
\Pr\left(v\le U_0\le (1+\kappa)v\right)\ge 1-\frac{\delta}{4K}.
\]
\end{lemma}

We first prove the implied statement of clean event $\cE$. To prove \eqref{eq:U_concentrate} for all $r\geq0$, we prove by induction. Assuming $v_i\leq U_{i,r}\leq(1+\kappa)v_i$, because $\cE^U_{i,r+1}$ implies $v_i\leq\widetilde U_{i,r+1}$, and recall $U_{i,r+1} = \min\{\widetilde U_{i,r+1}, U_{i,r}\}$, we get 
$$
v_i\leq U_{i,r+1}.
$$
On the other hand, we know 
$$
U_{i,r+1}=\min\{\widetilde U_{i,r+1}, U_{i,r}\}\leq U_{i,r} \leq U_{i,r}\leq(1+\kappa)v_i.
$$ 
Starting from $r=0$ given by $\cE_{\rm cal}$, we show \eqref{eq:U_concentrate} holds for all $r\geq0$.

Because we have proved $v_i\leq U_{i,r-1}$ holds for $r\geq1$, $\cE_{i,r}^p$ implies that
\begin{equation*}
\begin{aligned}
|\hat{p}_{i,r}-p_i| &\le \sqrt{2v_ih_{i,r}\log(2/\delta_r)} \\
&\leq \sqrt{2U_{i,r-1}h_{i,r}\log(2/\delta_r)}    \\
&\leq \sqrt{\frac{4U_{i,r-1}\log(2/\delta_r)}{t_{i,r}}}\leq\frac{\varepsilon_r}{4},
\end{aligned}
\end{equation*}
where the third inequality is given by $\cE_{i,r}^h=\{h_{i,r}\leq2/t_{i,r}\}$ and the final inequality is based on definition $t_{i,r}\geq {64U_{i,r-1}\log({2}/{\delta_r}})/{\eps_r^2}$, and then \eqref{eq:p_concentration} is proved.

We then prove the probability of the clean event $\cE$.
By Lemma~\ref{lem:calibration} and a union bound over the $K$ arms,
\[
\Pr(\cE_{\rm cal}^c)\le \frac{\delta}{4}.
\]
Since $t_{i,r}$ satisfies the lower bounds required in
Lemma~\ref{lem:leverage}, the lemma gives
\[
\Pr\left((\cE^h_{i,r})^c\right) \le \delta_r .
\]
For the variance-certificate update event $\cE^U_{i,r}$, given current fresh proxy variables
$\{W_{i,r,l}\}_{l=1}^{t_{i,r}}$, Proposition~\ref{prop:ols-concentration} gives ${(t_{i,r}-2)\hat v_{i,r}}/{v_i} \sim \chi^2_{t_{i,r}-2}$, then
\begin{equation*}
\begin{aligned}
&\Pr\left((\cE^U_{i,r})^c\right) \leq \Pr\left(v_i\geq \frac{\hat{v}_{i,r}}{1-2\sqrt{\frac{\log(1/\delta_r)}{t_{i,r}-2}}} \right)\\
=& \Pr\left( \frac{(t_{i,r}-2)\hat v_{i,r}}{v_i} \le (t_{i,r}-2) - 2\sqrt{(t_{i,r}-2)\log(1/\delta_r)} \right) \le \delta_r,
\end{aligned}
\end{equation*}
where the final inequality is based on the lower-tail part of Laurent-Massart inequality.
For the mean-estimation event $\cE_{i,r}^p$, given current fresh proxy variables $\{W_{i,r,l}\}_{l=1}^{t_{i,r}}$, the leverage $h_{i,r}$ is fixed, and Proposition~\ref{prop:ols-concentration} gives $\hat p_{i,r}-p_i \sim \mathcal{N}(0,v_i h_{i,r})$ and implies
\[
\Pr\left( |\hat p_{i,r}-p_i| > \sqrt{2v_i h_{i,r}\log\frac{2}{\delta_r}} \right) \le 2\exp\left(-\log\frac{2} {\delta_r}\right) = \delta_r .
\]

Finally, by a union bound and plugging into $\delta_r=\delta/(16Kr^2)$, we have
\begin{equation*}
\begin{aligned} 
\Pr(\cE^c) &\le \Pr(\cE_{\rm cal}^c) + \sum_{r=1}^{\infty}\sum_{i=1}^K \left(\Pr((\cE^h_{i,r})^c) + \Pr((\cE^U_{i,r})^c) + \Pr((\cE^p_{i,r})^c)\right)\\
&\leq \frac{\delta}{4} + 3K\sum_{r=1}^{\infty}\delta_r = \frac{\delta}{4} + 3K\sum_{r=1}^{\infty}\frac{\delta}{16Kr^2}\\
&\leq = \frac{\delta}{4} + \frac{3\delta}{16}\cdot2 = \frac{\delta}{4} + \frac{3\delta}{8} <\delta.
\end{aligned}
\end{equation*}
\end{proof}

\subsection{Proof of Lemma~\ref{lem:leverage}}\label{ssec:proof-lemma-leverage}
\begin{proof}
We have $\bar W\sim \mathcal{N}(0,1/s)$ and
\[
\sum_{l=1}^s(W_l-\bar W)^2\sim \chi^2_{s-1}.
\]
The Gaussian tail bound gives
\begin{equation}
\Pr\left(|\bar W| \leq \sqrt{\frac{2\log(4/\eta)}{s}}\right)
\geq 1- \frac{\eta}{2}.
\label{eq:|barWs|}
\end{equation}
Laurent-Massart inequality gives
\begin{equation}
\Pr\left(\sum_{l=1}^s(W_l-\bar W)^2 \geq (s-1)-2\sqrt{(s-1)\log(2/\eta)} \right) \geq 1- \frac{\eta}{2}.
\label{eq:SWW}
\end{equation}
Since $s\ge 1+16\log(2/\eta)$, we have $2\sqrt{(s-1)\log(2/\eta)}\leq {(s-1)}/{2}$. Thus, on an event of probability at least $1-\eta$,
\begin{equation*}
\begin{aligned}
\frac{\bar{W}^2}{\sum_{l=1}^s(W_l-\bar{W}_s)^2} &\leq \frac{2\log(4/\eta)/s}{(s-1)/2} \\
&\leq \frac{2\log(4/\eta)/s}{s/3} = \frac{6\log(4/\eta)}{s^2}\\
&\leq \frac{1}{s}.
\end{aligned}
\end{equation*}
where the first inequality is by plugging into high probability events in \eqref{eq:|barWs|} and \eqref{eq:SWW}, the second inequality is because $(s-1)/2\ge s/3$ based on $s\ge 3$, and the third inequality is based on definition $s\ge 8\log(4/\eta)$. Therefore, we have 
$$
\Pr(h_s\le 2/s) \geq 1-\eta.
$$
\end{proof}

\subsection{Proof of Lemma~\ref{lem:calibration}}\label{ssec:proof-lemma-calibration}

\begin{proof}
Fix an arm and omit the arm index. By Proposition~\ref{prop:ols-concentration}, ${(t_{\rm cal}-2)\hat v_0}/{v}\sim \chi^2_{t_{\rm cal}-2}$. Then, Laurent-Massart inequalities also states
\[
\Pr\left(\frac{(t_{\rm cal}-2)\hat v_0}{v}\ge (t_{\rm cal}-2)+2\sqrt{(t_{\rm cal}-2)\log(8K/\delta)}+2\log(8K/\delta)\right)\le e^{-\log(8K/\delta)}=\frac{\delta}{8K},
\]
which provides the high probability upper bound on $U_0/v$. Namely,
\[
\frac{\hat v_0}{v}\le 1+2\sqrt{\frac{\log(8K/\delta)}{t_{\rm cal}-2}}+2\frac{\log(8K/\delta)}{t_{\rm cal}-2}
\]
with probability at least $1-\delta/(8K)$. Let $y=\sqrt{{\log(8K/\delta)}/{(t_{\rm cal}-2)}}$, recall $U_0=\hat{v}_0/(1-2y)$, we have
\begin{equation*}
\begin{aligned}
\frac{U_0}{v}-1&\le \frac{1+2y+2y^2}{1-2y} - 1 =  \frac{4y+2y^2}{1-2y} \le 8y\le \kappa
\end{aligned}
\end{equation*}
where the second inequality is because the definition of $t_{\rm cal}$ implies $t_{\rm cal}-2\ge {64\log(8K/\delta)}/{\kappa^2}$ and further $y\leq\kappa/8\leq1/8$. Therefore we have
\begin{equation}
\Pr\left(U_0\le (1+\kappa)v \right) \geq 1-\frac{\delta}{8K}.
\label{eq:U<(1+k)v}
\end{equation}

At the same time, Laurent-Massart inequalities state that
\[
\Pr\left(\frac{(t_{\rm cal}-2)\hat v_0}{v}\le (t_{\rm cal}-2)-2\sqrt{(t_{\rm cal}-2)\log(8K/\delta)}\right)\le e^{-\log(8K/\delta)}=\frac{\delta}{8K},
\]
which provides the high probability lower bound on $\hat v_0/v$. Namely,
\[
\frac{\hat v_0}{v}\ge 1-2\sqrt{\frac{\log(8K/\delta)}{t_{\rm cal}-2}}
\]
with probability at least $1-\delta/(8K)$. Plug into the definition of $U_0$, we have $U_0\geq v$ with probability at least $1-\delta/(8K)$.
\begin{equation}
\Pr\left(U_0\geq v \right) \geq 1-\frac{\delta}{8K}.
\label{eq:v<U}
\end{equation}

Combining the \eqref{eq:U<(1+k)v} and \eqref{eq:v<U} with a union bound, we finish the proof.
\end{proof}

\section{Omitted Proofs}

\subsection{Proof of Proposition~\ref{prop:define-var-reduce}}
\begin{proof}
Since \(Y_{i,l}-m_i\) has mean zero, for any \(\beta\in\mathbb R\),
\[
\E\left[X_{i,l}-\beta(Y_{i,l}-m_i)\right] = \E[X_{i,l}]-\beta\E[Y_{i,l}-m_i] = p_i. 
\]
Thus the adjusted observation is unbiased for \(p_i\).

For the variance, using \(\Var(X_{i,l})=1\), \(\Var(Y_{i,l}-m_i)=\Var(Y_{i,l})=1\), and
\(\Cov(X_{i,l},Y_{i,l}-m_i)=\Cov(X_{i,l},Y_{i,l})=\rho_i\), we have
\[
\begin{aligned}
\Var\left(X_{i,l}-\beta(Y_{i,l}-m_i)\right)
&= \Var(X_{i,l}) +\beta^2\Var(Y_{i,l}-m_i) -2\beta\Cov(X_{i,l},Y_{i,l}-m_i)\\
&= 1+\beta^2-2\rho_i\beta. 
\end{aligned}
\]
Therefore the variance is minimized at \(\beta=\rho_i\), and the minimum variance is $1-\rho_i^2$.
\end{proof}

\subsection{Proof of Proposition~\ref{prop:ols-concentration}}
\begin{proof}
Fix an arm \(i\) and omit the arm index. Conditional on \(W_1,\ldots,W_s\), the regression model \eqref{eq:ols} is a fixed-design Gaussian linear model $X_l=p+\rho W_l+\epsilon_l$, where $\epsilon_l\stackrel{i.i.d.}{\sim}\mathcal{N}(0,v)$. Let
\[
Z=
\begin{pmatrix}
1 & W_1\\
\vdots & \vdots\\
1 & W_s
\end{pmatrix},
\qquad
\theta=
\begin{pmatrix}
p\\
\rho
\end{pmatrix},
\qquad
X=(X_1,\ldots,X_s)^\top .
\]
Then, \eqref{eq:ols} can be rewritten as
\[
X=Z\theta+\epsilon,
\qquad
\epsilon\mid W_1,\ldots,W_s\sim \mathcal{N}(0,vI_s).
\]
The OLS estimator is $\hat\theta=(Z^\top Z)^{-1}Z^\top X$. Therefore, conditional on \(W_1,\ldots,W_s\),
\[
\hat\theta-\theta = (Z^\top Z)^{-1}Z^\top \epsilon \sim \mathcal{N}\left(0,v(Z^\top Z)^{-1}\right).
\]
The intercept estimator \(\hat p\) is the first coordinate of \(\hat\theta\). Hence
\[
\hat p-p\mid W_1,\ldots,W_s \sim \mathcal{N}\left(0,v\left[(Z^\top Z)^{-1}\right]_{11}\right).
\]
where
\begin{equation*}
\begin{aligned}
\left[(Z^\top Z)^{-1}\right]_{11} &= \begin{bmatrix}\begin{pmatrix}
s & \sum_{l=1}^s W_l\\
\sum_{l=1}^s W_l & \sum_{l=1}^s W_l^2
\end{pmatrix}^{-1} \end{bmatrix}_{11}\\
&= \frac{\sum_{l=1}^s W_l^2}{s\sum_{l=1}^s W_l^2-\left(\sum_{l=1}^s W_l\right)^2} \\
&= \frac{1}{s} + \frac{\bar W^2}{\sum_{l=1}^s W_l^2-s\bar W^2} = h
\end{aligned}
\end{equation*}
Thus
\[
\hat p-p\mid W_1,\ldots,W_s\sim \mathcal{N}(0,vh).
\]

It remains to prove the residual variance statement. Let $P_Z=Z(Z^\top Z)^{-1}Z^\top$ be the projection matrix onto the column space of \(Z\). The residual vector is
\[
X-Z\hat\theta = (I-P_Z)X = (I-P_Z)\epsilon,
\]
because \((I-P_Z)Z\theta=0\). Conditional on \(W_1,\ldots,W_s\),
\[
\frac{1}{v}\|X-Z\hat\theta\|_2^2
=
\frac{1}{v}\epsilon^\top(I-P_Z)\epsilon.
\]
The matrix \(I-P_Z\) is symmetric and idempotent. Since the design has two linearly independent columns whenever \(S_{WW}>0\), we have
\[
\operatorname{rank}(I-P_Z)=s-2.
\]
For a Gaussian vector \(\epsilon\sim \mathcal{N}(0,vI_s)\), the quadratic form with a symmetric idempotent matrix of rank \(s-2\) satisfies
\[
\frac{1}{v}\epsilon^\top(I-P_Z)\epsilon\sim\chi^2_{s-2}.
\]
By the definition \eqref{eq:ols-var}, $\hat v=\|X-Z\hat\theta\|_2^2/(s-2)$, we obtain
\[
\frac{(s-2)\hat v}{v}\sim\chi^2_{s-2}.
\]
This completes the proof.
\end{proof}

\section{Repeated Proxy Calls}\label{sec:independent-proxy}
Consider a limited number of proxy query for $d$ times (where $d$ might be determined by the latency or concurrency in real world sequential decisions), the real reward $X_{i,l}$ is paired with a proxy vector $(Y_{i,l}^{(1)},\dots,Y_{i,l}^{(d)})$, which follow a multivariate gaussian distribution similar to \eqref{eq:proxy-model} as
\begin{equation}
\begin{pmatrix}
    X_{i,l}\\
    Y_{i,l}^{(1)}\\
    \vdots\\
    Y_{i,l}^{(d)}
\end{pmatrix}
\sim
\mathcal{N}\!\left(
    \begin{pmatrix}
        p_i\\
        m_i\mathbf 1_d
    \end{pmatrix},
    \begin{pmatrix}
        1 & \rho_i\mathbf 1_d^\top\\
        \rho_i\mathbf 1_d & \Sigma_d
    \end{pmatrix}
\right),\qquad \rho_i\in(-1,1),
\label{eq:proxy-model-vector}
\end{equation}
where $\Sigma_d$ is a positive semi-definite matrix with unit diagonal. Similar to the single proxy case, we can introduce $W_{i,l}^{(j)}=Y_{i,l}^{(j)}-m_i$ for all $j\in[d]$ then $\E[W_{i,l}^{(j)}]=0$ and $\Var(W_{i,l}^{(j)})=1$. Assuming $Y_{i,l}^{(1)},\dots,Y_{i,l}^{(d)}$ are conditional i.i.d given $X_{i,l}$, the repeated proxies 
\begin{equation*}
\bar W_{i,l}^{(d)}=\frac1d\sum_{j=1}^d W_{i,l}^{(j)},
\end{equation*}
can be reduced to a single averaged proxy, \ie, using all \(d\) proxies as linear control variates is equivalent to using their average  as a single proxy.
\begin{proposition}\label{prop:vector-proxy-reduction} 
Under the multivariate Gaussian model in \eqref{eq:proxy-model-vector}, and assume \(Y_{i,l}^{(1)},\ldots,Y_{i,l}^{(d)}\) are conditionally i.i.d. given \(X_{i,l}\) for all $j\in[d]$. Then, for every \(j\ne k\), $\Cov\left(Y_{i,l}^{(j)},Y_{i,l}^{(k)}\right)=\rho_i^2$, or equivalently 
$$
\Sigma_d=(1-\rho_i^2)I_d+\rho_i^2\mathbf 1_d\mathbf 1_d^\top.
$$
Moreover, using the single averaged proxy \(\bar W_{i,l}^{(d)}\) instead of the full vector. The effective correlation and residual variance are
$$
\rho_i^{(d)} = \frac{\sqrt d\,\rho_i}{\sqrt{1+(d-1)\rho_i^2}}, \qquad v_i^{(d)}=\frac{1-\rho_i^2}{1+(d-1)\rho_i^2}.
$$
\end{proposition}

\begin{proof}
We omit the arm $i$ and pull indices $l$, and write \(X=X_{i,l}\), \(Y^{(j)}=Y_{i,l}^{(j)}\), and \(W^{(j)}=Y^{(j)}-m_i\). Then \(\E[W^{(j)}]=0\), \(\Var(X)=\Var(W^{(j)})=1\), and $\Cov(X,W^{(j)})=\Cov(X,Y^{(j)})=\rho_i$ .

We first identify the covariance matrix of the repeated proxies. Extracting the principal sub-matrix indexed by $(X,W^{(j)},W^{(k)})$ from the covariance matrix in \eqref{eq:proxy-model-vector},  by Schur complement, we have
\begin{equation*}
\begin{aligned}
0 = \Cov(W^{(j)},W^{(k)}\mid X) 
&= \Cov(W^{(j)},W^{(k)}) - \Cov(W^{(j)},X)\Var(X)^{-1}\Cov(X,W^{(k)})\\
&= \Cov(W^{(j)},W^{(k)}) - \rho_i \cdot 1 \cdot \rho_i,
\end{aligned}
\end{equation*}
we obtain \(\Cov(W^{(j)},W^{(k)})=\rho_i^2\) for all $j\neq k$ and $j,k\in[d]$. Equivalently, for \(W=:(W^{(1)},\ldots,W^{(d)})^\top\),
\[
\Cov(W)=(1-\rho_i^2)I_d+\rho_i^2\mathbf 1_d\mathbf 1_d^\top.
\]

We now show that the full vector proxy is equivalent to the averaged proxy. Consider the best linear control variate based on the full vector \(W\),
\[
\alpha^* \in \arg\min_{\alpha\in\mathbb R^d} \Var\left(X-\alpha^\top W\right).
\]
Since \(\E[W]=0\), the solution is
\[
\alpha^*=\Cov(W)^{-1}\Cov(W,X).
\]
Recall that
\[
\Cov(W)=(1-\rho_i^2)I_d+\rho_i^2\mathbf 1_d\mathbf 1_d^\top,
\qquad
\Cov(W,X)=\rho_i\mathbf 1_d.
\]
By the Sherman-Morrison formula,
\[
\Cov(W)^{-1} = \frac{1}{1-\rho_i^2}I_d - \frac{\rho_i^2} {(1-\rho_i^2)(1+(d-1)\rho_i^2)} \mathbf 1_d\mathbf 1_d^\top.
\]
Therefore,
\[
\alpha^* = \left( \frac{1}{1-\rho_i^2}I_d - \frac{\rho_i^2} {(1-\rho_i^2)(1+(d-1)\rho_i^2)} \mathbf 1_d\mathbf 1_d^\top \right)\rho_i\mathbf 1_d = \frac{\rho_i}{1+(d-1)\rho_i^2}\mathbf 1_d. \]
Thus the optimal vector control term satisfies
\[
(\alpha^*)^\top W = \frac{\rho_i}{1+(d-1)\rho_i^2} \sum_{j=1}^d W^{(j)} = \frac{d\rho_i}{1+(d-1)\rho_i^2}\bar W^{(d)}.
\]
Therefore the optimal linear control variate based on the full proxy vector only depends on the averaged proxy \(\bar W^{(d)}\), which proves the claimed equivalence.
\end{proof}

Proposition~\ref{prop:vector-proxy-reduction} shows that the averaged proxy $\bar W_{i,l}^{(d)}$ is a sufficient summary of the whole proxy vector for the control-variate adjustment: the best linear adjustment built from $(Y_{i,l}^{(1)},\dots,Y_{i,l}^{(d)})$ coincides with the one built from their average alone. Consequently, the $d$-proxy observation reduces exactly to a single-proxy observation of the form \eqref{eq:proxy-model}, with the correlation and residual variance replaced by the effective quantities $\rho_i^{(d)}$ and $v_i^{(d)}$. Therefore, we are able to develop the algorithm and its analysis for a single proxy per pull without loss of generality, and recover the guarantee for any fixed $d$ by substituting $\rho_i^{(d)}$ and $v_i^{(d)}$ for $\rho_i$ and $v_i$.

\end{appendices}
\end{document}